\documentclass{article}

\usepackage{microtype}
\usepackage{graphicx}
\usepackage{subcaption}
\usepackage{booktabs} 

\usepackage[breaklinks]{hyperref}

\usepackage[preprint]{icml2026}

\usepackage{amsmath}
\usepackage{amssymb}
\usepackage{mathtools}
\usepackage{amsthm}

\usepackage{url}
\usepackage{xspace}
\usepackage{scalefnt}
\usepackage{tabularx}
\usepackage{thmtools} 
\usepackage{minitoc}
\usepackage{thm-restate}
\usepackage[capitalize,noabbrev]{cleveref}
\usepackage{tabularx}
\usepackage[textsize=tiny]{todonotes}
\usepackage{color}
\usepackage{enumerate}
\usepackage{amsfonts, amsmath, bm, amssymb}
\usepackage{dsfont}

\usepackage[capitalize,noabbrev]{cleveref}

\theoremstyle{plain}
\newtheorem{theorem}{Theorem}[section]
\newtheorem{proposition}[theorem]{Proposition}
\newtheorem{lemma}[theorem]{Lemma}

\theoremstyle{definition}
\newtheorem{definition}[theorem]{Definition}

\theoremstyle{remark}

\usepackage[textsize=tiny]{todonotes}

\icmltitlerunning{}

\newcommand{\DP}{\texttt{DP}\xspace}
\newcommand{\RDP}{\texttt{RDP}\xspace}
\newcommand{\DPSGD}{\texttt{DPSGD}\xspace}
\newcommand{\SGD}{\texttt{SGD}\xspace}
\newcommand{\LoRA}{\texttt{LoRA}\xspace}
\newcommand{\peft}{\texttt{PEFT}\xspace}
\newcommand{\ReLoRA}{\texttt{ReLoRA}\xspace}

\newcommand{\MIA}{\texttt{MIA}\xspace}

\newcommand{\thetav     }{\boldsymbol \theta     }

\begin{document}

\twocolumn[
  \icmltitle{\LoRA Provides Differential Privacy by Design via Random Sketching}



  \icmlsetsymbol{equal}{*}

  \begin{icmlauthorlist}
    \icmlauthor{Saber Malekmohammadi}{yyy}
    \icmlauthor{Golnoosh Farnadi}{yyy,xxx}
  \end{icmlauthorlist}

  \icmlaffiliation{yyy}{Mila - Quebec AI Institute, Montreal, Canada}
  \icmlaffiliation{xxx}{School of Computer Science, McGill University
and Université de Montréal,
Montreal, Canada}

  \icmlcorrespondingauthor{SM (work done while being a research intern at Mila)}{smsabermohammadi@gmail.com}

  \icmlkeywords{Machine Learning, ICML}

  \vskip 0.3in
]



\printAffiliationsAndNotice{}  

\section{Abstract}
Low-rank adaptation of language models has been proposed to reduce the computational and memory overhead of fine-tuning pre-trained language models. \texttt{LoRA} incorporates trainable low-rank matrices into some parameters of the pre-trained model, called \emph{adapters}. In this work, we show theoretically that the low-rank adaptation mechanism of \texttt{LoRA} is equivalent to fine-tuning \emph{adapters} with noisy batch gradients, with the noise variance being a decreasing function of adaptation rank ($r$). Motivated by this understanding, we prove \emph{inherent} differential privacy for \texttt{LoRA} when adaptation matrices $\mathbf{A}_\ell$ are frozen. We show that various factors, e.g., the adaptation rank and batch size, affect the guaranteed privacy level. Our findings provide useful insights into \LoRA and uncovers the reason behind the robustness of models fine-tuned with \LoRA to privacy attacks.
   
\section{Introduction}

In the recent years, we have witnessed vast and rapid advances in large language models (\texttt{LLM}s) \citep{touvron2023llamaopenefficientfoundation,openai2024gpt4technicalreport, zhang2022optopenpretrainedtransformer,zeng2023glm130bopenbilingualpretrained,gemma3team2025}. Modern \texttt{LLM}s often contain large number of parameters \citep{openai2024gpt4technicalreport, stable_diffusion}, making their fine-tuning computationally prohibitive. Parameter-efficient fine-tuning (\peft; \citet{peft, Ding2023ParameterefficientFO}) reduces the space complexity by updating a subset of parameters while freezing the base model \citep{houlsby, zaken}. \LoRA \citep{lora} is such a \peft method that updates only a subset of parameters, called \emph{adapters}. For each adapter $\mathbf{W}_\ell$ located in layer $\ell$ of a pretrained model, \LoRA introduces two trainable low-rank matrices $\mathbf{A}_\ell$ and $\mathbf{B}_\ell$, which are often \emph{randomly} initialized. The low-rank restriction considerably reduces the number of trainable parameters, yet \LoRA often matches full-parameter fine-tuning on downstream tasks \citep{lora}. There are also other variants of \LoRA. For instance, \ReLoRA \citep{relora} enables high-rank fine-tuning through iterative low-rank updates, by repeated merging and re-initializing the low-rank component matrices $\mathbf{A}_\ell$ and $\mathbf{B}_\ell$. Just like \LoRA, this process also introduces further \emph{randomness} into the fine-tuning procedure. 

Fine-tuning data often happens to contain privacy-sensitive data samples, rising the threat for privacy attacks to the fine-tuned models \citep{NEURIPS2019_zhu, zhao2020idlgimproveddeepleakage, Nasr2018ComprehensivePA, liu2024precuriousinnocentpretrainedlanguage}. These issues motivate privacy-aware fine-tuning of language models. Differential privacy (\DP) is the de facto standard for ensuring rigorous data privacy in machine learning \cite{Dwork2006OurDO, mcsherrytalwar, dpjournal}. A widely used algorithm to achieve \DP is \DPSGD \citep{abadi}, which employs Gaussian mechanism \cite{abadi, Dwork2006OurDO} and limits the influence of any single data sample on the final trained model by adding a calibrated noise to clipped batch gradients. However, fine-tuning the entire existing parameters of a pre-trained model using \DPSGD induces high space complexities \citep{abadi}. As a remedy, to achieve both data privacy and low space complexity, \DPSGD can be applied only to the \emph{adapters}, while freezing the rest of the base model. However, the additional noise injection by \DPSGD induces a tradeoff between utility loss and privacy enhancement of the fine-tuned model.

Interestingly, without a clear understanding of its underlying reason, some prior works observed that models that are solely fine-tuned with \LoRA show more robustness to privacy attacks compared to when they are fine-tuned with \SGD \cite{tahseen2025, liu2024precuriousinnocentpretrainedlanguage}. This observations motivate us to answer the following important and unanswered question rigorously:

\begin{center}
\emph{Why does fine-tuning with \LoRA yield to robustness to privacy attacks?}
\end{center}

Indeed, we show theoretically that a commonly used realization of \LoRA secretly outputs differentially private fine-tuned models \emph{by design}, meaning that as a \peft method, not only does \LoRA reduce the computational overheads of fine-tuning and preserve accuracy, but also it provides \emph{a provable data privacy}. The main insight in our privacy proof is that, when the component matrices $\{\mathbf{A}_\ell\}$ are frozen, \LoRA updates parameters by applying a Gaussian random sketching \citep{pmlr-v98-sheffet19a, lev2025gaussianmixingmechanismrenyi} to batch gradients. 
This is an interesting and insightful finding about \LoRA. The highlights of our contributions are as follows:

\begin{itemize}
    \item We show that fine-tuning \emph{adapters} with \LoRA is equivalent to fine-tuning them with noisy batch gradients (\cref{eq:param_update_noise}). 

    \item We show that the variance of the injected noise decreases with adaptation rank $(r)$ (\cref{lem:noise_var}). Also, its distribution is close to a Gaussian distribution in terms of the total variation distance ( \cref{lem:total_var}).

    \item We show that \LoRA randomly sketches batch gradients w.r.t. \emph{adapters} during fine-tuning time. Accordingly, we prove that it provides \DP by design (\Cref{lem:DPofRS} and \Cref{lem:DPofLora}). The guaranteed privacy level increases as the adaptation rank $r$ decreases.

    \item  Finally, we show experimentally that low-rank adaptation with a smaller rank $r$ provides more robustness to membership inference attacks (\Cref{sec:exps}), confirming our theoretical findings. 
    
\end{itemize}


\section{Related Work}
Some recent works observed that models fine-tuned with \LoRA show robustness to privacy attacks without a clear understanding of its root reason \citep{liu2024precuriousinnocentpretrainedlanguage, tahseen2025}. \citet{tahseen2025} observed that \LoRA can reduce unintended memorization in centralized and federated fine-tuning settings up to a factor of $10$. Overlooking the aforementioned \emph{inherent randomness} in low-rank adaptation methods such as \LoRA and \texttt{ReLoRA}, some works \citep{Liu2023DifferentiallyPL, tahseen2025} proposed privatized variants of \LoRA in federated and centralized settings. These methods adopt the Gaussian mechanism to the component matrices $\mathbf{A}_\ell$ and $\mathbf{B}_\ell$ of \LoRA, without identifying the \emph{inherent randomness} existing in \LoRA itself.

The work in \citep{flora} identified that the dynamics of \LoRA can be approximated by a random matrix projection. Under the condition that component matrices $\{\mathbf{A}_\ell\}$ are fixed and only $\{\mathbf{B}_\ell\}$ are tuned, this estimate is no longer an approximation but an exact equality. This variant of \LoRA has been shown to have more stability and less communication/computation overheads in privacy-preserving federated \LoRA settings \citep{sun_ICLR2024}, which makes it a practical and useful realization of \LoRA.

There is a plethora of research works exploring the usage of random projections, e.g., Johnson–Lindenstrauss (\texttt{JL}) transform \citep{arora2024differentiallyprivategeneralizedlinear, JL}. Some prior works studied the privacy-preserving characteristics of noisy Gaussian sketching for private linear regression \citep{Showkatbakhsh2018PrivacyUtilityTO, pmlr-v98-sheffet19a, lev2025gaussianmixingmechanismrenyi}: Given a matrix $\mathbf{Q}(\mathcal{D}) \in \mathbb{R}^{n \times m}$ computed on a privacy-sensitive dataset $\mathcal{D}$, noisy Gaussian sketching constructs a noisy compressed representation of it, as follows

\begin{align}\label{eq:random_sketching}
    &\mathcal{M}(\mathbf{Q}(\mathcal{D})) := \mathbf{Q}(\mathcal{D})\mathbf{A}^\top + \mathbf{Z}, \nonumber \\
    &\text{where}~~ \mathbf{A} \sim \mathcal{N}(\mathbf{0}, \sigma_A^2 \mathbb{I}_{r \times m}), ~~\mathbf{Z} \sim \mathcal{N}(\mathbf{0}, \sigma_g^2 \mathbb{I}_{n \times r}).
\end{align}

Very similar to the above noisy Gaussian sketching, \citet{li2025sketchedgaussianmechanismprivate} introduced sketched Gaussian mechanism (\texttt{SGM}) for when $\mathbf{Q}(\mathcal{D}) \in \mathbb{R}^{1 \times m}$ is a row vector (not a matrix). Using R\'enyi \DP tools \citep{Mironov_2017}, they presented a joint analysis of \texttt{SGM}’s overall privacy guarantee, which is significantly more flexible and sharper compared to isolated privacy analysis of sketching and Gaussian mechanism. \citet{7138618} showed that the mutual information
between $\mathbf{Q}(\mathcal{D})$ and its sketched version $\mathcal{M}(\mathbf{Q}(\mathcal{D}))$ can not be too large. Leveraging \texttt{RDP} framework, \citet{lev2025gaussianmixingmechanismrenyi} recently derived a tight privacy guarantee for the noisy Gaussian sketching, improving upon the bounds obtained in earlier works. However, in their privacy analysis, they consider zero-out notion of neighborhood for the matrix $\mathbf{Q}(\mathcal{D})$: all columns of $\mathbf{Q}(\mathcal{D})$ have a bounded norm and only one column can change to a zero column. In contrast, in our privacy analysis, we consider a more general notion of neighborhood on the underlying dataset $\mathcal{D}$ itself: replacing a sample in $\mathcal{D}$ results in a change in $\mathbf{Q}(\mathcal{D})$ with a bounded Frobenius norm. We use this more general privacy analysis of random sketching for our privacy analysis of \LoRA.

\citet{fang2025federatedsketchingloraflexible} proposed a sketched version of \LoRA for federated fine-tuning settings with heterogeneous communication and computation constraints for edge devices. In each round, each device receives a sampled sketching matrix from the server to selectively update submatrices of global \LoRA modules $\mathbf{A}_\ell$ and $\mathbf{B}_\ell$ maintained by the server. In contrast, we leverage the overlooked sketching mechanism that exists in \LoRA dynamics by design to prove a guaranteed privacy for \LoRA. Our findings justify robustness of models fine-tuned solely with \LoRA to privacy attacks.

\section{Preliminaries}\label{app:background}

\paragraph{\textbf{Basic Notation.}} 
We denote random matrices with boldface uppercase letters. For a matrix $\mathbf{A}$, $\|\mathbf{A}\|_F$ denotes its Frobenius norm. The $k \times k$ identity matrix is denoted by $\mathbb{I}_k$ and $\mathcal{N}(0, \mathbb{I}_{k1 \times k2})$ is $k1 \times k2$ matrix with \emph{i.i.d} standard normal entries. The fine-tuning dataset $\mathcal{D}$ has size $N$ and $b$ is the used batch size. There are $E$ epochs of fine-tuning and $T=EN/b$ gradient steps. The data batch used at the $t$-th gradient step is denoted by $\mathcal{B}^t$, and has batch loss $\mathcal{L}^t(\mathbf{W^t})$ and batch gradient $\nabla_{\mathbf{W}^t}\mathcal{L}^t(\mathbf{W}^t) = \frac{1}{b}\sum_{i \in \mathcal{B}^t}\nabla_{\mathbf{W}^t}l(x_i; \mathbf{W}^t)$ w.r.t. adapter $\mathbf{W}^t$, where $l$ is the used training loss. 

\paragraph{\textbf{Differential Privacy.}} 
In machine learning, the following definition of differential privacy (\DP) is commonly used:

\begin{definition}[($\epsilon,\delta$)-\DP \citep{Dwork2006OurDO}]
\label{def:epsilondeltadp}
A randomized mechanism $\mathcal{M}:\mathcal{A}\to \mathcal{R}$ with domain $\mathcal{A}$ and range $\mathcal{R}$ satisfies $(\epsilon,\delta)$-\DP if for any two adjacent input datasets $\mathcal{D} \simeq \mathcal{D}'$, which differ only by a single record\footnote{We consider adjacency by replacement, where a record in $\mathcal{D}$ is replaced by another record in $\mathcal{D}'$.}, and for any measurable subset of outputs $\mathcal{S} \subseteq \mathcal{R}$ it holds that
\begin{align*}
    \texttt{Pr}[\mathcal{M}(\mathcal{D})\in \mathcal{S}] \leq e^{\epsilon} \texttt{Pr}[\mathcal{M}(\mathcal{D}')\in \mathcal{S}]+\delta.
\end{align*}
\end{definition}

We also adopt a secondary and stronger notion of \DP, which is called R\'enyi \DP:

\begin{definition}[($\alpha, \epsilon$)-\texttt{RDP} \citep{Mironov_2017}]
\label{def:rdp}
A randomized mechanism $\mathcal{M}:\mathcal{A}\to \mathcal{R}$ with domain $\mathcal{A}$ and range $\mathcal{R}$ satisfies $(\alpha, \epsilon)$-\texttt{RDP} with order $\alpha>1$ if for any two adjacent input datasets $\mathcal{D} \simeq \mathcal{D}'$, it satisfies
\begin{align}
    D_{\alpha}\big(\mathcal{M}(\mathcal{D})||\mathcal{M}(\mathcal{D}')\big) \leq \epsilon,
\end{align}
\end{definition}
where $D_{\alpha}(P||Q)$ is the R\'enyi divergence between distributions $P$ and $Q$:

\begin{align}
    D_{\alpha}(P||Q) := \frac{1}{\alpha - 1} \log \mathbb E_{x\sim P} \bigg[ \bigg(\frac{P(x)}{Q(x)}\bigg)^{\alpha - 1}\bigg].
\end{align}

In contrast to $(\alpha, \epsilon)-$\DP, \RDP bounds the moments of the likelihood ratio of events induced by $\mathcal{D}$ and $\mathcal{D}'$. One can translate \RDP to $(\epsilon, \delta)$-\DP, as follows:

\begin{proposition}[Converting \texttt{RDP} \citep{ Cannon_kamath_2020}]If a mechanism $\mathcal{M}$ satisifes $(\alpha, \epsilon(\alpha))$-\texttt{RDP}, then for any $\delta>0$, it satisfies $(\epsilon(\delta), \delta)$-\DP, where 
\begin{align}
    \epsilon(\delta) = \inf_{\alpha>1} \epsilon(\alpha) + \frac{1}{\alpha - 1} \log \big(\frac{1}{\alpha \delta}\big) + \log \big(1-\frac{1}{\alpha}\big).
\end{align}
\label{prop:rdptodp}
\end{proposition}

The postprocessing property of \DP states that if mechanism $\mathcal{M}$ satisfies a \DP definition, $\mathcal{G}\circ \mathcal{M}$ enjoys the same privacy for any post-processing function $\mathcal{G}$ \citep{Mironov_2017, Dwork2006OurDO}.

Zero concentrated \texttt{DP} (\texttt{zCDP})
is another relaxed definition of differential privacy. Being $\rho$-\texttt{zCDP} ($\rho>0$) is equivalent to being $(\alpha, \rho \alpha)$-\texttt{RDP} simultaneously for all $\alpha > 1$ \citep{zcdp}. Accordingly, we have the truncated cocentrated \DP (\texttt{tCDP}) \citep{tcdp}: A mechanism is $(\rho, w)$-\texttt{tCDP} ($w>1$) if it is $(\alpha, \rho \alpha)$-\texttt{RDP} for the limited range $\alpha \in (1,w)$. We can use the following proposition to translate \texttt{tCDP} to $(\epsilon, \delta)$-\DP:

\begin{proposition}[Converting \texttt{tCDP} \citep{tcdp}] If a mechanism $\mathcal{M}$ satisifes $(\rho, w)$-\texttt{tCDP}, then for any $\delta>0$, it satisfies $(\epsilon(\delta), \delta)$-\DP, where 
\begin{align}
    \epsilon(\delta) = \begin{dcases}
     \rho + 2\sqrt{\rho \cdot \log(1/\delta)}~~~~~\text{if}~\log(1/\delta) \leq (w-1)^2 \rho\\
      \rho \cdot w + \frac{\log(1/\delta)}{w-1}
     ~~~~~~~~~\text{if}~\log(1/\delta) > (w-1)^2 \rho
\end{dcases}
\end{align}
\label{prop:tcdptodp}
\end{proposition}

Gaussian mechanism uses additive noise to achieve $(\epsilon, \delta)-$\DP. More specifically, it randomizes the output of a non-private computation $f(\mathcal{D}) \in \mathbb{R}^p$ on a dataset $\mathcal{D}$ as $\mathbf{G}_{\sigma}f(\mathcal{D}) := f(\mathcal{D})+\mathcal{N}(\mathbf{0},\sigma_{\texttt{\DP}}^2 \mathbb{I}_p)$. 
The variance $\sigma_{\texttt{\DP}}^2$ of the additive noise is calibrated to the $\ell_2$ sensitivity $\Delta_2f = \max_{\mathcal{D} \simeq \mathcal{D}'}\parallel f(\mathcal{D})-f(\mathcal{D}')\parallel_2$ of $f$. Gaussian mechanism is $\frac{\Delta_2f^2}{2 \sigma_{ \texttt{\DP}}^2}$-\texttt{zCDP}. It is also $(\epsilon, \delta)$-\texttt{DP} for all $\epsilon, \delta \in (0,1)$ that $ \sigma_{\texttt{\DP}}^2 \geq \Delta_2 f\sqrt{2\ln (1.25/\delta)}/\epsilon$. Gaussian mechanism has been used in \DPSGD algorithm \citep{abadi} for private training of models by randomizing intermediate data-dependent computations, e.g., batch gradients: At the $t$-th gradient update step on a current model $\thetav \in \mathbb R^p$, \DPSGD computes the following clipped noisy batch gradient:
\begin{align}
    \tilde{g}(\thetav) &= \frac{1}{b}\bigg[ \Big(\sum_{i \in \mathcal{B}^t} \Bar{g}_{i}(\thetav)\Big) + \mathcal{N}(\mathbf{0}, \sigma_{ \texttt{\DP}}^2 \mathbb{I}_p)\bigg] \nonumber \\
    &=  \frac{1}{b}\sum_{i \in \mathcal{B}^t} \Bar{g}_{i}(\thetav) + \mathcal{N}(\mathbf{0}, \frac{\sigma_{ \texttt{\DP}}^2}{b^2} \mathbb{I}_p),
    \label{eq:noisy_sg}
\end{align}
where $\Bar{g}_{i}(\thetav) = \texttt{clip}(g_i, c)$, $g_i$ is the sample gradient computed on the sample $i$ in the batch $\mathcal{B}^t$ and $c$ is a clipping threshold. For a given vector $\mathbf{v}$, $\texttt{clip}(\mathbf{v}, c) =  \min\{\|\mathbf{v}\|, c\} \cdot \frac{\mathbf{v}}{\|\mathbf{v}\|}$. Also,  $\sigma_{ \texttt{\DP}}=c\cdot z$, where $z$ is the noise scale that should be used by \DPSGD and can be computed by using a privacy accountant, e.g., the moments accountant \citep{abadi} or R\'enyi accountant \citep{Mironov_2017}.

\section{Dynamics of low-rank adaptation}

We now start to study the dynamics of \LoRA. \Cref{fig:lora} summarizes the findings in this section. In order to update a pre-trained \emph{adapter} $\mathbf{W}_{\ell} \in \mathbb{R}^{n \times m}$ located in layer $\ell$ of a pre-trained model, \LoRA incorporates low-rank decomposition matrices $\mathbf{B}_{\ell} \in \mathbb{R}^{n\times r}$ and $\mathbf{A}_{\ell} \in \mathbb{R}^{r\times m}$, with $r \ll \min \{n,m\}$, and performs the following forward pass in an adaptation layer $\ell$:

\begin{align}\label{eq:forward}
    y = (\mathbf{W}_{\ell}+\mathbf{B}_{\ell}\mathbf{A}_{\ell})x = \mathbf{W}_{\ell}x + \mathbf{B}_{\ell}\mathbf{A}_{\ell}x,
\end{align}

\begin{figure*}[t]
\centering
 \includegraphics[width=1.95\columnwidth]{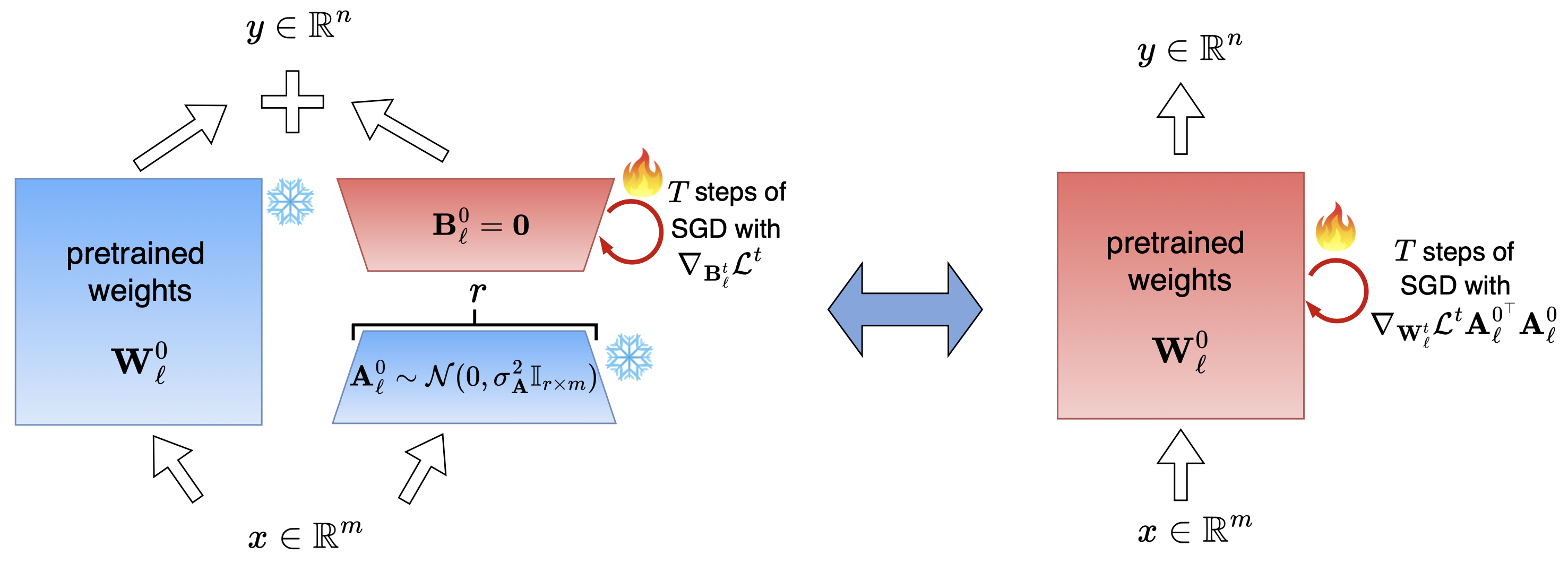}
 \caption{Low-rank adaptation of \LoRA in an adaptation layer $\ell$ initialized with pretrained weight $\mathbf{W}_\ell^0$, and component matrices $\mathbf{A}_\ell^0$ and $\mathbf{B}_\ell^0$. \textbf{Left:} \LoRA with a frozen component matrix $\mathbf{A}_\ell$ and fine-tuning component matrix $\mathbf{B}_\ell$ with $T$ SGD steps. \textbf{Right:} Fine-tuning the pretrained adapter matrix $\mathbf{W}_\ell^0$ with $T$ SGD steps with \emph{noisy} batch gradients $\nabla_{\mathbf{W}_\ell^t} \mathcal{L}^t \mathbf{A}_\ell^{0^\top} \mathbf{A}_\ell^0$. }
\label{fig:lora}
\end{figure*}

where $x \in \mathbb{R}^m$ and $y \in \mathbb{R}^n$ are the input and the pre-activation output of the current adaptation layer, respectively. For instance, there are 12 self-attention layers in GPT-2 (small) model and when fine-tuning it with \LoRA, $\mathbf{W}_{\ell}$ can be considered to be the query/key/value parameter in each of those 12 layers, which have dimension $(m,n)=(768,768)$. When back-propagating, \LoRA freezes $\mathbf{W}_\ell$ and calculates the gradients w.r.t. only $\mathbf{A}_{\ell}$ and $\mathbf{B}_{\ell}$, which can be found as follows:




\begin{align}\label{eq:down_projectionA}
    \color{blue}{\frac{\partial \mathcal{L}^t}{\partial \mathbf{A}_{\ell}}}  &= \frac{\partial \mathbf{B}_{\ell} \mathbf{A}_{\ell}}{\partial \mathbf{A}_{\ell}} \cdot \frac{\partial \mathcal{L}^t}{\partial \mathbf{B}_{\ell}\mathbf{A}_{\ell}} = \mathbf{B}_{\ell}^\top \cdot \frac{\partial \mathcal{L}^t}{\partial y} \cdot \frac{\partial y}{\partial \mathbf{B}_{\ell}\mathbf{A}_{\ell}} \nonumber \\
&= \mathbf{B}_{\ell}^\top \cdot \frac{\partial \mathcal{L}^t}{\partial y} \cdot x^\top = \color{blue}{\mathbf{B}_{\ell}^\top (\nabla_{\mathbf{W}_{\ell}} \mathcal{L}^t)},
\end{align}
 and 
\begin{align}\label{eq:down_projectionB}
\color{blue}{\frac{\partial \mathcal{L}^t}{\partial \mathbf{B}_{\ell}}} &=  \frac{\partial \mathcal{L}^t}{\partial \mathbf{B}_{\ell}\mathbf{A}_{\ell}} \cdot \frac{\partial \mathbf{B}_{\ell}\mathbf{A}_{\ell}}{\partial \mathbf{B}_{\ell}} = \frac{\partial \mathcal{L}^t}{\partial y} \cdot \frac{\partial y}{\partial \mathbf{B}_{\ell}\mathbf{A}_{\ell}} \cdot \mathbf{A}_{\ell}^\top \nonumber \\
&= \frac{\partial \mathcal{L}^t}{\partial y} \cdot x^\top \cdot \mathbf{A}_{\ell}^\top = \color{blue}{(\nabla_{\mathbf{W}_{\ell}} \mathcal{L}^t) \mathbf{A}_{\ell}^\top}.
\end{align}

Hence, $\frac{\partial \mathcal{L}^t}{\partial \mathbf{A}_{\ell}} \in \mathbb{R}^{r\times m}$ and $\frac{\partial \mathcal{L}^t}{\partial \mathbf{B}_{\ell}} \in \mathbb{R}^{n\times r}$. In fact, \emph{\LoRA down-projects the batch gradient $\nabla_{\mathbf{W}_{\ell}} \mathcal{L}^t$ from $\mathbb{R}^{n \times m}$ to a lower dimension, and updates the matrices $\mathbf{A}_{\ell}$ and $\mathbf{B}_{\ell}$ with the resulting projections}. 
Let $\mathbf{A}_{\ell}^0$, $\mathbf{B}_{\ell}^0$ and $\mathbf{W}_{\ell}^0$ denote the initial values of the matrices $\mathbf{A}_{\ell}$, $\mathbf{B}_{\ell}$ and the \emph{adapter} $\mathbf{W}_\ell^0$ (coming from the pretrained model). A common practice is to initialize $\mathbf{B}_\ell$ with an all-zero matrix so that at the beginning of fine-tuning, the original pretrained $\mathbf{W}_{\ell}^0$ does not change. 
A common initialization of $\mathbf{A}_{\ell}^0$ is sampling from Gaussian distribution $\mathcal{N}(0, \sigma_{\mathbf{A}}^2)$. 
Having the insights from \Cref{eq:down_projectionA} and (\ref{eq:down_projectionB}) in mind, we state the following lemma based on the derivations in \citep{flora}.

\begin{restatable}{lemma}{lorafinalforwardpass}
\label{lem:lora_final_forwardpass}
Initializing $\mathbf{B}_{\ell}$ with a zero matrix, let \LoRA update matrices $\mathbf{A}_{\ell}$ and $\mathbf{B}_{\ell}$ with \SGD and learning rate $\eta$ for $T$ gradient steps. Let $\mathbf{W}_{\ell}^t = \mathbf{W}_{\ell}^0 + \mathbf{B}_{\ell}^t\mathbf{A}_{\ell}^t$ be the \textbf{equivalent forward pass parameter} at the end of step $t$. The final equivalent parameter $\mathbf{W}_{\ell}^T$ can be written as: 
\begin{align}\label{eq:noisy_sgd}
    \mathbf{W}_{\ell}^T &= \mathbf{W}_{\ell}^0 + \mathbf{B}_{\ell}^T \mathbf{A}_{\ell}^T \nonumber \\
    &= \mathbf{W}_{\ell}^0 - \eta \bigg(\sum_{t=0}^{T-1} \nabla_{\mathbf{W}_{\ell}^t} \mathcal{L}^t \mathbf{A}_{\ell}^0{^\top} \mathbf{A}_{\ell}^0 \bigg ) + \mathcal{O}(\eta^3 ),
\end{align}

where the term $\nabla_{\mathbf{W}_{\ell}^t} \mathcal{L}^t$ in the sum is the batch gradient at time step $t$ w.r.t. $\mathbf{W}_{\ell}^t$, that would be obtained if we wanted to ``fine-tune $\mathbf{W}_{\ell}^t$ at time step $t$ with \SGD''.
\end{restatable}

In the next section, we show that the update in \cref{eq:noisy_sgd} to the forward pass parameter $\mathbf{W}_\ell^0$ is equivalent to ``fine-tuning'' $\mathbf{W}_\ell^0$ with a \emph{noisy} version of batch gradient $\nabla_{\mathbf{W}_{\ell}^t} \mathcal{L}^t$ at every time step $t$.

\section{\LoRA injects random noise into adapters' batch gradients}

Based on \Cref{lem:lora_final_forwardpass}, the update in \cref{eq:noisy_sgd} to the initial forward pass parameter $\mathbf{W}_{\ell}^0$ can be rewritten as:

\begin{align}\label{eq:param_update_noise}
    \mathbf{W}_{\ell}^T \approx \mathbf{W}_{\ell}^0 - \eta \sum_{t=0}^{T-1} \bigg[ \underbrace{\nabla_{\mathbf{W}_{\ell}^t} \mathcal{L}^t}_{\textit{ batch gradient}} +\underbrace{ \nabla_{\mathbf{W}_{\ell}^t} \mathcal{L}^t \big(\mathbf{A}_{\ell}^0{^\top} \mathbf{A}_{\ell}^0-\mathbb{I}_m\big)}_{\textit{noise} ~\in~ \mathbb{R}^{n\times m}}\bigg]
\end{align}

The second term in the sum represents a noise term introduced by the low-rank adaptation. Also, \emph{when $\{\mathbf{A}_\ell\}_{\ell=1}^L$ are frozen at their initialized values $\{\mathbf{A}_\ell^0\}_{\ell=1}^L$, the approximation above turns into an equality}. As mentioned earlier \LoRA with frozen $\{\mathbf{A}_\ell^0\}_{\ell=1}^L$ provides some benefits including: 1. less computational overhead while preserving model utility \citep{flora}, and 2. less communication overhead as well as more stability in federated \LoRA settings \citep{sun_ICLR2024}. Therefore, \emph{``low-rank adaptation'' of pretrained adapter $\mathbf{W}_{\ell}^0$ by introducing low-rank matrices $\{\mathbf{A}_{\ell}^0\}_{\ell=0}^L$ (frozen) and $\{\mathbf{B}_{\ell}\}_{\ell=0}^L$ (initialized with zeros) is equivalent to ``fine-tuning'' it with noisy batch gradients}. We are now particularly interested in the behavior of this noise term. In Appendix \ref{app:kaiming}, we also consider another initialization of $\{\mathbf{A}_{\ell}\}_{\ell=0}^L$ called Kaiming-uniform initialization, and show that it results in a \emph{biased} noise that almost cancels out the batch gradient $\nabla_{\mathbf{W}_{\ell}^t} \mathcal{L}^t$ at each time step $t$, leading to slow convergence and low utility. This justifies the adoption of the widely used Gaussian initialization for $\{\mathbf{A}_\ell\}_{\ell=1}^L$, which is studied next.

\subsection{Analysis of the additive noise with Gaussian initialization of $\{\mathbf{A}_\ell\}_{\ell=1}^L$}

When each of the 
$r$ columns of $\mathbf{A}_{\ell}^0{^\top} \in \mathbb{R}^{m\times r}$ is an 
$m$-dimensional Gaussian random variable, $\mathbf{A}_{\ell}^0{^\top} \mathbf{A}_{\ell}^0$ follows a Wishart distribution with $r$ degrees of freedom \citep{Wishart1}, which is the multivariate generalization of the chi-squared distribution. Therefore, for any 
$ \mathbf{q}\in \mathbb{R}^{1 \times m}$, $ \mathbf{q} \cdot (\mathbf{A}_{\ell}^0{^\top} \mathbf{A}_{\ell}^0-\mathbb{I}_m)$  is a weighted sum of multiple chi-squared random variables, which implies that the result follows a Gaussian distribution approximately, according to the Central Limit Theorem (CLT) \citep{bhattacharya2016course}. Therefore, we first state and prove the following lemma concerning the noise term in \cref{eq:param_update_noise}.

\begin{restatable}{lemma}{noisevar}\label{lem:noise_var}
Let $\mathbf{A}_{\ell} \in \mathbb{R}^{r\times m}$ be a matrix with \emph{i.i.d} entries sampled from $\mathcal{N}(0, \sigma_{\mathbf{A}}^2)$. Given a fixed $\mathbf{q}\in \mathbb{R}^{1 \times m}$, distribution of noise term $\mathbf{q}\cdot (\mathbf{A}_{\ell}^\top \mathbf{A}_{\ell}-\mathbb{I}_m) \in \mathbb{R}^{1 \times m}$ approaches the $m$-dimensional Gaussian distribution $\mathcal{N}\big((r\sigma_{\mathbf{A}}^2 - 1)~\mathbf{q}, ~r\sigma_{\mathbf{A}}^4 (\|\mathbf{q}\|_2^2 ~ \mathbb{I}_m + \mathbf{D})\big)$, as $m$ increases. $\mathbf{D}$ is a diagonal matrix with $\mathbf{q}_i^2$ as its $i$-th diagonal element. When $\sigma_{\mathbf{A}}^2 = 1/r$, it approcahes $\mathcal{N}\big(\mathbf{0}, ~ \frac{1}{r}(\|\mathbf{q}\|_2^2 ~ \mathbb{I}_m + \mathbf{D})\big)$, i.e., an \textbf{unbiased} noise with a variance decreasing with $r$.
\end{restatable}

Hence, distribution of $\mathbf{G} = \nabla_{\mathbf{W}_{\ell}^t} \mathcal{L}^t \cdot (\mathbf{A}_{\ell}^\top \mathbf{A}_{\ell} - \mathbb{I}_m)$ (as seen in \Cref{eq:param_update_noise}) also approaches a Gaussian distribution, where $\mathbf{G}_{i,:}$ ($i$-th row of $\mathbf{G}$, $1\leq i\leq n$) approaches to $\mathcal{N}\big((r\sigma_{\mathbf{A}}^2-1)[\nabla_{\mathbf{W}_{\ell}^t} \mathcal{L}^t]_{i,:}, ~ r\sigma_{\mathbf{A}}^4 \big(\|[\nabla_{\mathbf{W}_{\ell}^t} \mathcal{L}^t]_{i,:}\|_2^2 ~ \mathbb{I}_m + \mathbf{D}\big) \big)$, where $[\nabla_{\mathbf{W}_{\ell}^t} \mathcal{L}^t]_{i,:}$ is the $i$-th row of $\nabla_{\mathbf{W}_{\ell}^t} \mathcal{L}^t$ and $\mathbf{D}\in \mathbb{R}^{m\times m}$ is the diagonal matrix with $\mathbf{D}_{j,j} = [\nabla_{\mathbf{W}_{\ell}^t} \mathcal{L}^t]_{i,j}^2$. 

Having $\sigma_{\mathbf{A}}^2 \ll 1/r$ will lead to the noise term in \Cref{eq:param_update_noise} being unbiased: its distribution will be centered around $(r\sigma_{\mathbf{A}}^2 - 1) \nabla_{\mathbf{W}_{\ell}^t} \mathcal{L}^t \approx - \nabla_{\mathbf{W}_{\ell}^t} \mathcal{L}^t$ with a small variance. In other words, the noise term almost cancels the batch gradients at each time step, slowing down the fine-tuning procedure (see also Appendix \ref{app:kaiming}). Hence, a common practice for Gaussian initialization is to set $\sigma_{\mathbf{A}}^2 = 1/r$, which leads to an \emph{unbiased} noise injection into the batch gradients. Consistent with our above findings, \citet{kalajdzievski2023rankstabilizationscalingfactor} proposed rank-stabilized \LoRA (\texttt{rsLoRA}), an approach that enhances \texttt{LoRA}’s performance in
high-rank scenarios by setting $\sigma_{\mathbf{A}}^2 = 1/r$. Also, \citet{shuttleworth2025loravsfinetuningillusion} demonstrated recently that only with this design, \texttt{rsLoRA} can approach the performance of fine-tuning \emph{adapters} with \SGD, as the rank $r$ increases. Consistently, \Cref{fig:rplot} also shows that only under this setting of $\sigma_{\mathbf{A}}^2 = 1/r$, the noise term in \Cref{eq:param_update_noise} is unbiased and its magnitude decreases as $r$ increases. Therefore, \emph{we fix $\sigma_{\mathbf{A}}^2 = 1/r$ in the remaining sections. Hence, according to \Cref{lem:noise_var}, \textbf{the smaller the adaptation rank, the larger the magnitude of the noise injected into the batch gradients $\{\nabla_{\mathbf{W}_\ell^t}\mathcal{L}^t\}_{\ell=1}^L$ w.r.t. the adapters}}.

In Appendix \ref{app:complete_analysis}, we complete our analysis of the injected noise. Specifically, in Appendix \ref{app:gaussian_deviation}, we bound the total variation distance between the injected noise distribution and a pure Gaussian distribution with the same variance for a limited value of $m$ (as opposed to $m \to \infty$ in \Cref{lem:noise_var}).

The observations above motivate us to investigate if \LoRA with its inherent noise injection into batch gradients provides a provable differential privacy by design.  

\begin{figure}[t]
\centering
\includegraphics[width=0.95\columnwidth]{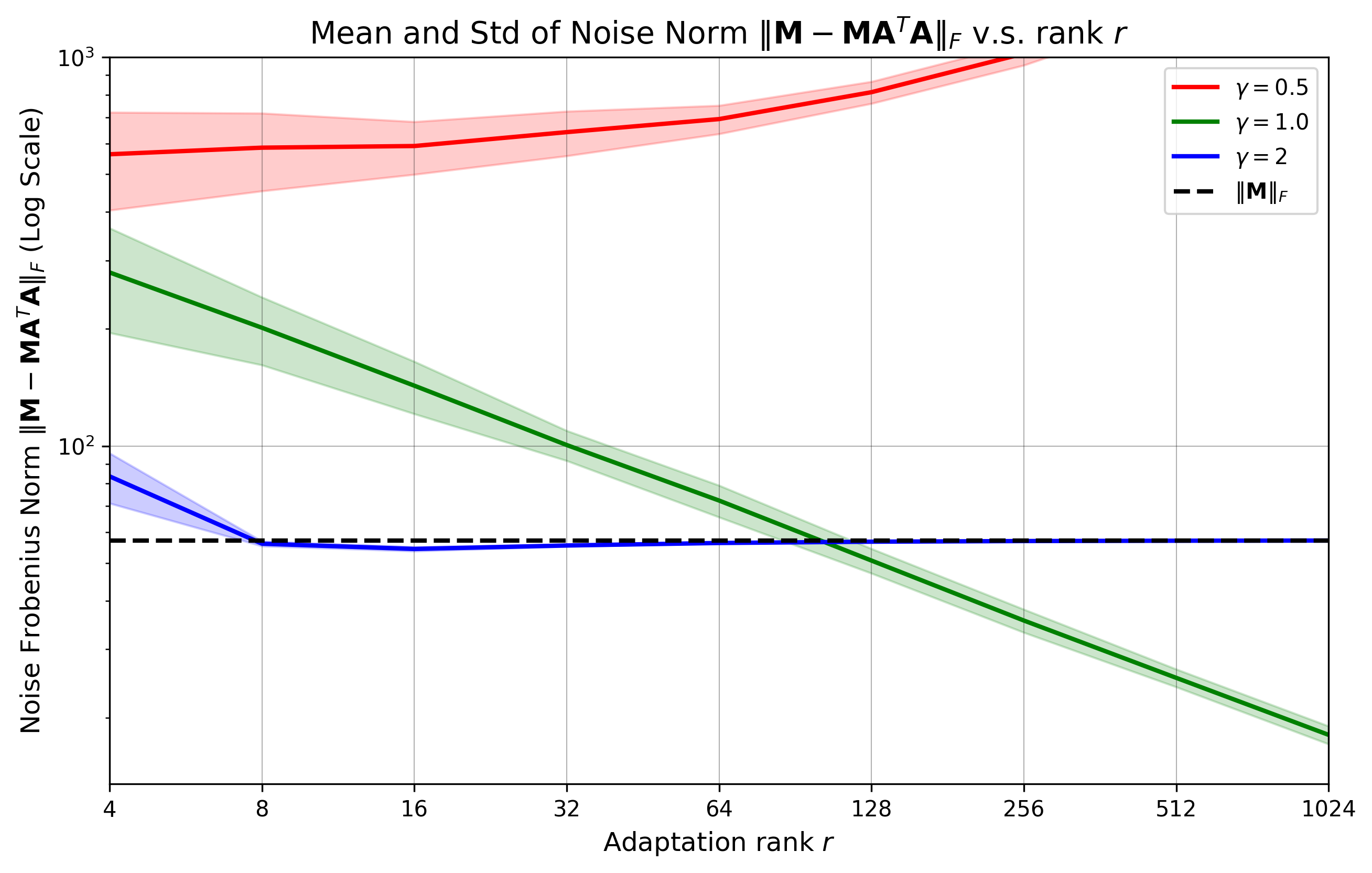}
    \caption{Magnitude (Frobenius norm) of the noise term $\mathbf{M}-\mathbf{M}\mathbf{A}^\top\mathbf{A}$ (as seen in \Cref{eq:param_update_noise}): The elements of the random matrix $\mathbf{M}\in \mathbb{R}^{100\times 100}$ are sampled uniformly from $(0,1)$. The elements of $\mathbf{A}\in \mathbb{R}^{r\times 100}$ are sampled from $\mathcal{N}(0,\sigma_{\mathbf{A}}^2)$, where $\sigma_{\mathbf{A}}^2 = 1/r^\gamma$. The results are based on 5 realizations of $\mathbf{A}$. Only for $\gamma=1$, i.e., $\sigma_{\mathbf{A}}^2 = 1/r$, the noise term is unbiased and its variance drops as $r$ increases, confirming the results in \Cref{lem:noise_var}.}\label{fig:rplot}
\end{figure}

\section{\LoRA provides \DP via randomly sketching batch gradients}
We now prove that \LoRA with Gaussian and frozen $\{\mathbf{A}_\ell^0\}_{\ell=0}^L$ is differentially private \emph{by design}. Let us recall the random sketching in \Cref{eq:random_sketching}, which constructs a noisy compressed representation $\mathcal{M}(\mathbf{Q}(\mathcal{D})) \in \mathbb{R}^{n \times r}$ of $\mathbf{Q}(\mathcal{D}) \in \mathbb{R}^{n \times m}$ computed on a privacy-sensitive dataset $\mathcal{D}$.
We first prove that, under some assumptions on the matrix $\mathbf{Q}(\mathcal{D})$, the result of the random sketching is provably differentially private.

\begin{restatable}{lemma}{DPofRS}\label{lem:DPofRS}
For a dataset $\mathcal{D}$ and its neighboring dataset $\mathcal{D}'$ (by replacement), let us assume that matrix $\mathbf{Q}(\mathcal{D})\in \mathbb{R}^{n \times m}$ ($m\geq n$) satisfies norm bound $\|\mathbf{Q}\|_F \leq B_F$ and sensitivity bound $\|\mathbf{Q}(\mathcal{D}) - \mathbf{Q}(\mathcal{D'})\|_F \leq B_S$. Let us assume that there is a lower bound for the eigen values of $\mathbf{Q}(\mathcal{D})\mathbf{Q}(\mathcal{D})^\top$: $\lambda_{min}(\mathbf{Q}(\mathcal{D})\mathbf{Q}(\mathcal{D})^\top)\geq \bar{\lambda}>0$, and also $\Gamma = \frac{\bar{\lambda}+\sigma_g^2/\sigma_A^2}{B_C}>1$, with $B_C = 2B_F B_S$. Then, the random sketching in \Cref{eq:random_sketching} satisfies $(\alpha, \epsilon(\alpha))$-\RDP for any $\alpha \in (1, \Gamma)$, and 

\begin{align}
    \epsilon(\alpha) = \frac{r \tau}{2(\alpha -1 )}\bigg[\alpha \log \big(1-\Gamma^{-1}\big) - \log \big(1-\alpha \Gamma^{-1}\big) \bigg],
\end{align}
\end{restatable}
 where integer $\tau \leq n$. The function $\epsilon(\alpha)$ is a non-negative function, as the term appearing in brackets is of the form $f(t;\alpha)= \alpha\log(1-t) -\log(1-\alpha t)$, which is non-decreasing and non-negative for every $\alpha>1$ and $\alpha t <1$. Our proof improves the result by \citet{lev2025gaussianmixingmechanismrenyi} for the privacy analysis of noisy Gaussian sketching: The work assumes $\mathbf{Q}(\mathcal{D})$ and $\mathbf{Q}(\mathcal{D}')$ differ in only one column with bounded norm $B_S$. We have instead replaced this assumption with the more general assumptions $\|\mathbf{Q}(\mathcal{D}) - \mathbf{Q}(\mathcal{D'})\|_F \leq B_S$ and $\|\mathbf{Q}(\mathcal{D})\|_F \leq B_F$. Similar to the bound by \citet{lev2025gaussianmixingmechanismrenyi}, our bound is tight: when $\mathbf{Q}(\mathcal{D})$ and $\mathbf{Q}(\mathcal{D}')$ differ in one column of $\mathbf{Q}(\mathcal{D})$ (with a bounded column norm $B_S$), and the column happens to be the eigenvector corresponding to the minimum eigenvalue of $\mathbf{Q}(\mathcal{D})\cdot \mathbf{Q}(\mathcal{D})^\top$, then we can replace the term $B_C = 2 B_F B_S$ with $B_C = B_S^2$ and set $\tau=1$, and the bound will be achieved with equality (see Appendix \ref{app:proofofdpofrs}). 

Accodring to the definition of $\Gamma$ in the lemma, we observe that the additive noise with variance $\sigma_g^2$ in \Cref{eq:random_sketching} can be looked at as artificially increasing the minimum eigen value bound $\bar{\lambda}$, which results in a smaller \RDP bound $\epsilon(\alpha)$ and a stronger privacy guarantee. Furthermore, for a fixed $\sigma_g^2$, the \RDP upper bound $\epsilon(\alpha)$ decreases linearly with the sketching dimension $r$, resulting in a stronger privacy guarantee as expected.

\paragraph{\textbf{Sketch of the proof.}} We first show that, for a given dataset $\mathcal{D}$, the $r$ columns of $\mathcal{M}(\mathbf{Q}(\mathcal{D})) \in \mathbb{R}^{n \times r}$ are \emph{i.i.d} multi-variate Gaussian variables with mean zero and covariance $\mathbf{\Sigma}_1 = \sigma_A^2 \mathbf{Q}(\mathcal{D})\mathbf{Q}(\mathcal{D})^\top + \sigma_g^2 \mathbb{I}_n$. Then, we show that  $D_{\alpha}(\mathcal{M}(\mathbf{Q}(\mathcal{D}))||\mathcal{M}(\mathbf{Q}(\mathcal{D}')))$ is a monotonic function of $ \sigma_A^2 \cdot \big(\lambda_{min}(\mathbf{\Sigma_1})\big)^{-1}\cdot \big\|\mathbf{Q}(\mathcal{D})\mathbf{Q}(\mathcal{D})^\top - \mathbf{Q}(\mathcal{D}')\mathbf{Q}(\mathcal{D}')^\top\big\|_F$, which can be upperbounded by $\Gamma^{-1}$. More details are deferred to Appendix \ref{app:proofofdpofrs}. 

The results above can be directly applied to \LoRA. Remember from \Cref{eq:param_update_noise} that when $\mathbf{A}_{\ell}$ is frozen, we have:

\begin{align}
    \mathbf{W}_{\ell}^T &= \mathbf{W}_{\ell}^0 + \mathbf{B}_{\ell}^T \mathbf{A}_{\ell}^0  = \mathbf{W}_{\ell}^0 + \bigg(-\eta \sum_{t=0}^{T-1}\nabla_{\mathbf{B}_{\ell}^t} \mathcal{L}^t\bigg) \mathbf{A}_{\ell}^0 \nonumber \\
    &= \mathbf{W}_{\ell}^0 + \bigg(-\eta \sum_{t=0}^{T-1}\nabla_{\mathbf{W}_{\ell}^t} \mathcal{L}^t \mathbf{A}_{\ell}^0{^\top}\bigg) \mathbf{A}_{\ell}^0,
\end{align}

where the last equality came from \Cref{eq:down_projectionB}. Comparing the last term with \Cref{eq:random_sketching}, we can clearly observe that, \LoRA in fact performs a random sketching on cumulative batch gradient $\mathcal{G}_\ell(\mathcal{D}):= \sum_{t=0}^{T-1}\nabla_{\mathbf{W}_{\ell}^t} \mathcal{L}^t$ using the Gaussian matrix $\mathbf{A}_{\ell}^0{^\top} \sim \mathcal{N}(\mathbf{0}, \frac{1}{r}\mathbb
I_{m\times r})$. In the following, we prove that this sketching mechanism provides differential privacy for \LoRA. For the generality of our analysis, let us assume that we also add some Gaussian noise $\mathbf{Z}_{\ell} \sim \mathcal{N}(\mathbf{0}, \sigma_g^2 \mathbb{I}_{n \times r})$ to the cumulative batch gradient $\sum_{t=0}^{T-1}\nabla_{\mathbf{B}_{\ell}^t} \mathcal{L}^t$ to get to the final noisy matrix $\tilde {\mathbf{B}}_{\ell}^T$ 
(obviously, setting $\sigma_g^2=0$ results in the original vanilla \LoRA). We will have:

\begin{align}\label{eq:param_update_noisy_lora}
    \tilde{\mathbf{W}}_{\ell}^T &= \mathbf{W}_{\ell}^0 + \tilde {\mathbf{B}}_{\ell}^T \mathbf{A}_{\ell}^0 = \mathbf{W}_{\ell}^0 -\eta \bigg(\sum_{t=0}^{T-1}\nabla_{\mathbf{B}_{\ell}^t} \mathcal{L}^t + \mathbf{Z}_{\ell}\bigg) \mathbf{A}_{\ell}^0 \nonumber \\
    &= \mathbf{W}_{\ell}^0 - \eta \bigg( \color{blue}\sum_{t=0}^{T-1}\nabla_{\mathbf{W}_{\ell}^t} \mathcal{L}^t \mathbf{A}_{\ell}^0{^\top} + \mathbf{Z}_\ell \color{black} \bigg) \mathbf{A}_{\ell}^0 \nonumber
    \\
    &= \mathbf{W}_{\ell}^0 - \eta \bigg(\color{blue}{\mathcal{G}_\ell(\mathcal{D}) \mathbf{A}_{\ell}^0{^\top} + \mathbf{Z}_\ell}\color{black} \bigg) \mathbf{A}_{\ell}^0.
\end{align}
We state and prove the following lemma for the \LoRA's set of output adapters $\{\tilde{\mathbf{W}}_\ell^T\}_{\ell=1}^L$. 


\begin{table*}[t]
  \centering
  \renewcommand{\arraystretch}{1.2}
  \caption{Membership inference attack on GPT-2. \LoRA (with frozen and trainable $\{\mathbf{A}_\ell\}$) shows robustness to \MIA in terms of AUC and TPR. \texttt{FA-LoRA} is \LoRA with frozen component matrices $\{\mathbf{A}_\ell\}$.}
  
  \resizebox{1.99\columnwidth}{!}{%
  \begin{tabular}{|p{5.5cm}|c|c|c|c|c|c|c|c|c|c|c|c|}
    \hline
    \footnotesize\textbf{Dataset} & \multicolumn{4}{c|}{\footnotesize\textbf{Enron}}& \multicolumn{4}{c|}{\footnotesize\textbf{Pubmed}}& \multicolumn{4}{c|}{\footnotesize\textbf{PTB}}\\
    \hline
    \footnotesize\textbf{Evaluation Criterion} & \footnotesize\textbf{Utility}&\multicolumn{3}{c|}{\footnotesize\textbf{MIA Success Metrics}}&  \footnotesize\textbf{Utility}&\multicolumn{3}{c|}{\footnotesize\textbf{MIA Success Metrics}}& \footnotesize\textbf{Utility}&\multicolumn{3}{c|}{\footnotesize\textbf{MIA Success Metrics}}\\
    \cline{1-13}
    \footnotesize \textbf{Metric}& \footnotesize PPL$\downarrow$ & \footnotesize AUC$\downarrow$ & \footnotesize @FPR 10\%$\downarrow$ & \footnotesize @FPR 1\%$\downarrow$ & \footnotesize PPL$\downarrow$ & \footnotesize AUC$\downarrow$ & \footnotesize @FPR 10\%$\downarrow$ & \footnotesize @FPR 1\%$\downarrow$ & \footnotesize PPL$\downarrow$ & \footnotesize AUC$\downarrow$ & \footnotesize @FPR 10\%$\downarrow$ & \footnotesize @FPR 1\%$\downarrow$\\
    \hline \hline
    \footnotesize\textbf{Full fine-tuning} & 18.49 & 0.876 & 66.31 & 13.02 & 16.64& 0.929& 80.67& 6.48& 27.70& 0.963 & 90.54 & 65.84\\ \hline
    \footnotesize\textbf{\LoRA with $(r,E,b)=(16, 20, 8)$} & 19.66& 0.757 & 40.56& 8.08& 18.82& 0.816& 50.43& 2.68& 31.18 & 0.961 & 89.97 & 59.88\\ \hline
    \footnotesize\textbf{\texttt{FA-LoRA} with $(r,E,b)=(16, 20, 8)$} & 21.90& 0.614 & 16.76& 4.49& 20.85& 0.546 & 14.93& 2.37& 37.93&0.952 & 82.23 & 55.58\\ \hline
    \footnotesize\textbf{\texttt{FA-LoRA} with $(r,E,b)=(16, 40, 16)$} & 21.05& 0.598 & 15.86& 4.09& 19.95& 0.486 & 12.86& 2.02& 35.86&0.872 & 79.63 & 48.23\\ \hline
  \end{tabular}}
  
\label{tab:exp_results}  
\end{table*}

\begin{restatable}{lemma}{DPofLoRA}\label{lem:DPofLora} Consider the described \LoRA realization with rank $r$ on the $L$ pretrained  adapters $\{\mathbf{W}_{\ell}^0\}_{\ell=1}^L$ $\in \mathbb{R}^{n \times m}$ ($m\geq n$) by performing $T$ \SGD steps with batch size $b$ on the initial component matrices $\{\mathbf{B}_\ell^0\}_{\ell=0}^L$. Let us assume train loss $l$ is $\beta$-smooth, and for every sample data point $x_i \in \mathcal{D}$, the sample gradient $\nabla_{\mathbf{W}_{\ell}^t}l(x_i; \mathbf{W}_{\ell}^t) \in \mathbb{R}^{n \times m}$ satisfies norm bound $\|\nabla_{\mathbf{W}_{\ell}^t} l(x_i; \mathbf{W}_{\ell}^t)\|_F \leq C$. Also, for the cumulative batch gradient $\mathcal{G}_\ell(\mathcal{D})=\sum_{t=0}^{T-1}\nabla_{\mathbf{W}_{\ell}^t} \mathcal{L}^t \in \mathbb{R}^{n \times m}$, we have $\lambda_{min}(\mathcal{G}_\ell(\mathcal{D}) \cdot \mathcal{G}_\ell(\mathcal{D})^\top)\geq \bar{\lambda}$ for some $\bar{\lambda}>0$. Finally, $\Gamma := \frac{\eta \beta}{4}\cdot \frac{b( \bar{\lambda} + r\sigma_g^2)}{ C^2 T ((1+ \eta \beta)^{T}-1)}>1$. If only the final model with the updated adapters $\{\tilde{\mathbf{W}}_\ell^T\}_{\ell=1}^L$ are released, the low-rank adaptation procedure satisfies $(\alpha, \epsilon(\alpha))$-\RDP for any $\alpha \in (1, \Gamma)$, and 
\begin{align}\label{eq:rdpbound}
    \epsilon(\alpha) = \frac{r L \tau}{2(\alpha -1 )}\bigg[\alpha \log \big(1-\Gamma^{-1}\big) - \log \big(1-\alpha \Gamma^{-1}\big) \bigg],
\end{align}
\end{restatable}

where integer $\tau \leq n$. We note that the norm bound $C$ on the sample gradients always exists if the train loss function $l$ is Lipschitz continuous. For every $\alpha \in (1, \Gamma)$, function $f(t; \alpha) = \alpha \log (1-t)-\log(1-\alpha t)$ in the brackets is a non-decreasing function of $t$ for $\alpha t <1$ and $f(0; \alpha)=0$. Setting $\sigma_g^2=0$ results in $\Gamma = \frac{\eta \beta}{4}\cdot \frac{b\bar{\lambda}}{ C^2 T ((1+ \eta \beta)^{T}-1)}$ and proves the privacy of the original vanilla \LoRA (with no additive noise). Given that $\Gamma$ is sufficiently large, the bound $\epsilon(\alpha)$ increases at most linearly with $\alpha$ for a range of $\alpha$, resulting in the following \texttt{tCDP} guarantee for \LoRA:

\begin{restatable}{corollary}{tCDPofLoRA}\label{cor:tCDPofLora} Considering the setup of \Cref{lem:DPofLora}, if $\Gamma>5/2$, then \LoRA satisfies $(rLn/2\Gamma^2, 2\Gamma/5)$-\texttt{tCDP}.
\end{restatable}

We can now directly use the above \texttt{tCDP} guarantee and \Cref{prop:tcdptodp} to derive the $(\epsilon, \delta)-$\DP guarantee of \LoRA: If $\Gamma>5/2$, \LoRA is $(rLn/2\Gamma^2 + \frac{\sqrt{2rLn\log(1/\delta)}}{\Gamma} , \delta)-$\DP. The dependence of the guarantee on $r$, $b$ and $T$ is of particular interest: the \texttt{DP} bound increases linearly with the adaptation rank $r$ (smaller rank leads to more noisy batch gradients). Furthermore, $\Gamma$ increases linearly with $b$ and, at a fixed number of gradient steps $T=EN/b$, using larger batch size $b$ results in a smaller \DP bound, suggesting more privacy. Finally, $\Gamma$ decreases linear-exponentially with $T$ and as expected, increasing $T$ leads to less privacy.

\section{Experiments}\label{sec:exps}
In this section, we perform membership inference attacks (\MIA) on some models fine-tuned with \LoRA to show how low-rank adaptation indeed provides robustness to them. 

\par{\textbf{Datasets:}} We run our experiments on three datasets with confidential properties: Penn Treebank (PTB) \citep{marcusetal-1993-building}, Enron \citep{Klimt2004TheEC} and Pubmed \citep{cohan2018discourseawareattentionmodelabstractive}. We split each dataset into three parts and use them as train, validation and auxilliary datasets.

\par{\textbf{Models:}}  We perform our experiments with GPT-2 (12-layer, 125M parameters, vocab size 50257) on two Nvidia A100 GPUs. We apply \LoRA \citep{lora} with default $r=16$ to the query parameters in the 12 attention layers of GPT-2. We also consider fine-tuning with \SGD (no low-rank adaptation) as another baseline. We use $E=20$ epochs for fine-tuning with the default learning rates $\{$\texttt{1e-5, 5e-4}$\}$ (with a linear scheduler) for full fine-tuning and low-rank adaptation with \LoRA, respectively.

\par{\textbf{Membership Inference Attacks:}} We use a calibrated membership score for membership inference \citep{carlini2022membershipinferenceattacksprinciples, mireshghallah2022memorizationnlpfinetuningmethods, sablayrolles2019whiteboxvsblackboxbayes, ye2022enhancedmembershipinferenceattacks}. More specifically, we train a reference model $\theta_{ref}$ on the auxiliary dataset and use the following signal 
for inferring train set membership of a sample $x$:

\begin{align}
    I_{\theta}(x) := \mathbb{I}[\mathcal{L}(x;\theta) - \mathcal{L}(x; \theta_{ref}) < \gamma]. 
\end{align}
Threshold $\gamma$ is set to the highest value for which the false positive rate will be lower than $\alpha \in \{10\%, 1\%\}$ and we report the corresponding true positive rates (TPR).

\par{\textbf{Metrics:}} To measure utility, we use the perplexity on the validation set (PPL). For measuring the effectiveness of \MIA, we use AUC and TPR$@$FPR $\alpha$ for $\alpha \in \{10\%, 1\%\}$.


\begin{figure*}[h]
    \centering
    \setlength{\columnsep}{0pt}\includegraphics[width=0.65\columnwidth,height=3.5cm]{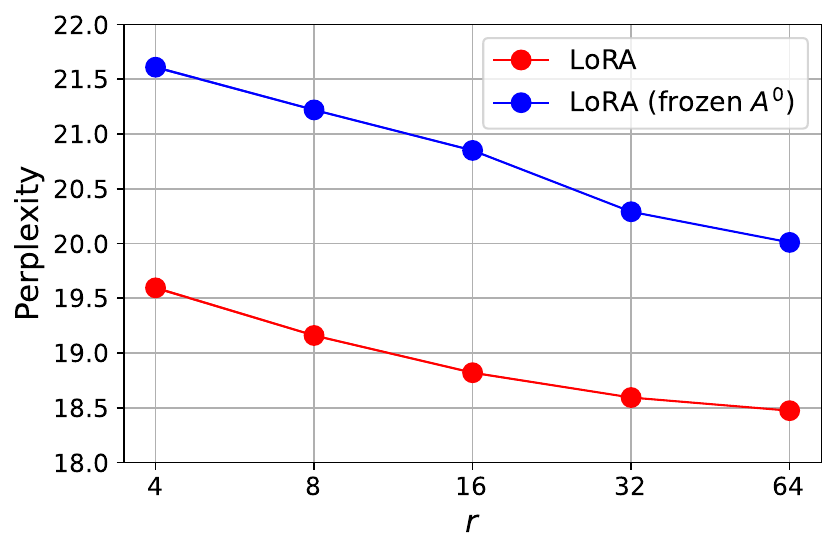} \includegraphics[width=0.65\columnwidth,height=3.5cm]{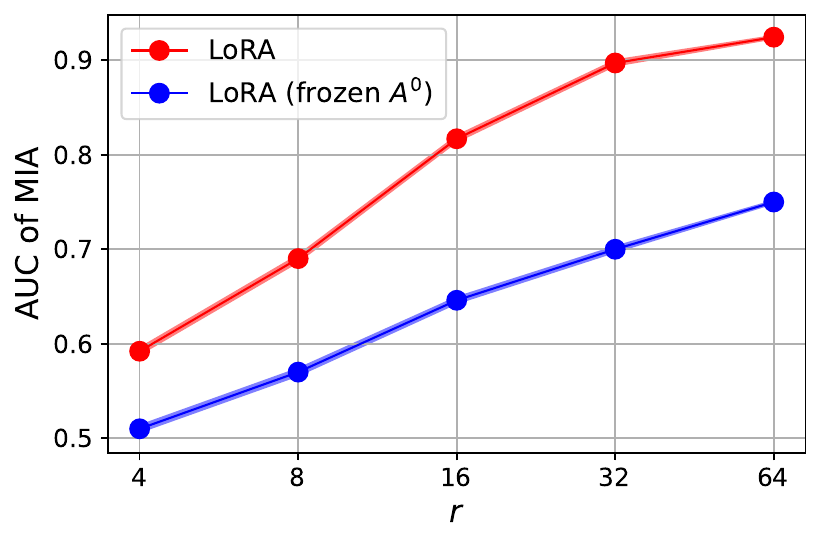}  
    \caption{The effect of adaptation rank $r$ averaged over three different data splits of Pubmed. \textbf{Left:} The utility of the fine-tuned model increases as $r$ increases. \textbf{Right:} The AUC of  \MIA attacks increases as $r$ increases, showing less robustness to attacks for lager $r$. 
    }
    \label{fig:ablation_r}
\end{figure*}

\begin{figure*}[h]
    \centering
    \setlength{\columnsep}{0pt}\includegraphics[width=0.65\columnwidth,height=3.5cm]{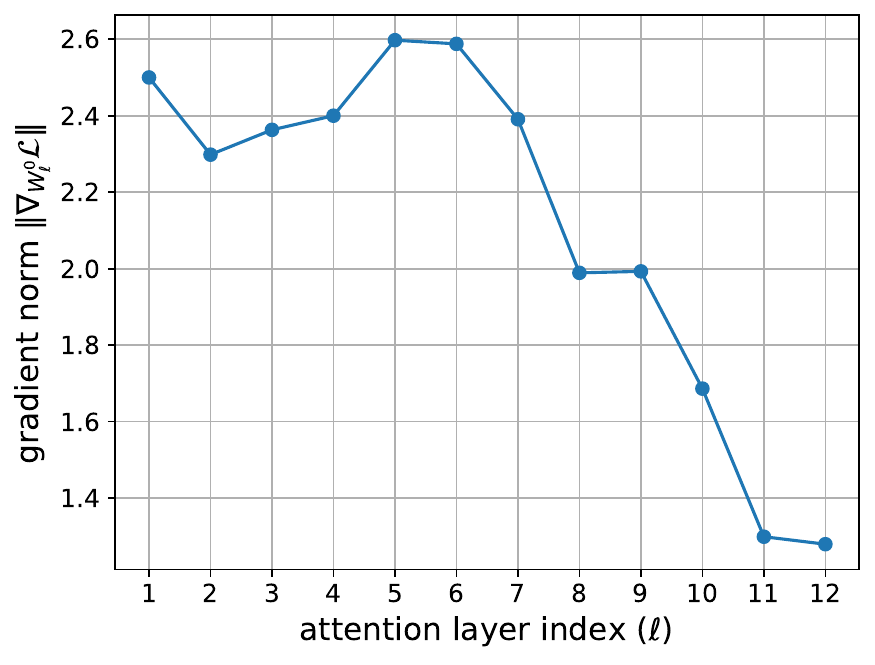} \quad \quad \includegraphics[width=0.6\columnwidth,height=3.5cm]{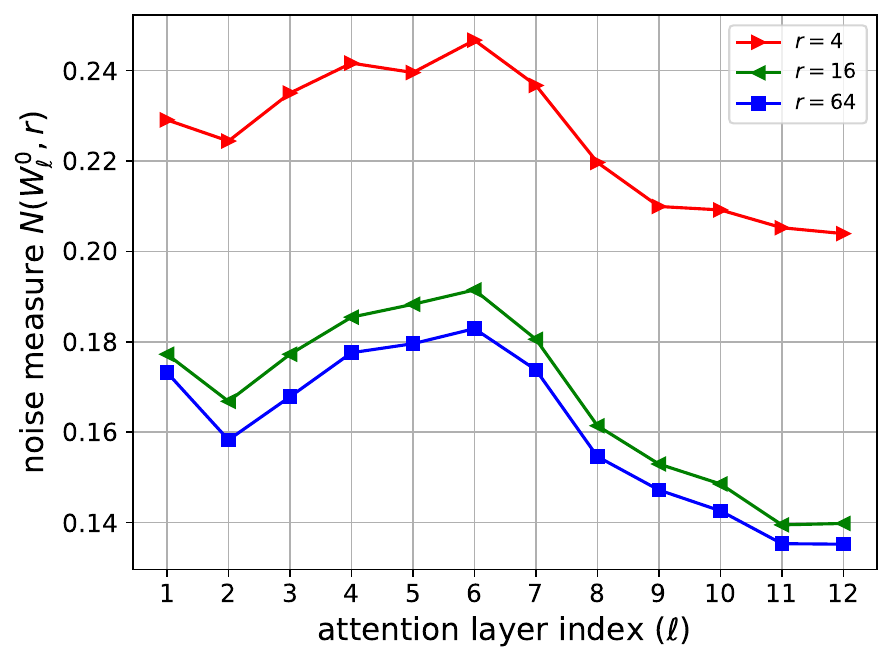}
    \caption{Experimental results obtained on Pubmed dataset showing that, unlike \DPSGD, \LoRA adds noise to batch gradient w.r.t. the adapter parameters \emph{non-uniformly}. \textbf{Left:} the norm of the batch gradient w.r.t. the \emph{adapters} in different layers of the GPT-2 model. Batch gradient w.r.t. the \emph{adapters} located in the fifth and sixth layers have the largest norms. \textbf{Right:} The noise measure $N(\mathbf{W}_{\ell}^0,r)$ evaluated at different layers (indexed by $\ell$) of the model, showing that those layers with larger batch gradient w.r.t. their \emph{adapter} experience a large noise injection by \LoRA. Also, as the adaptation rank $r$ increases the amount of noise injected to each layer decreases.  
    }
    \label{fig:layernorms}
\end{figure*}

\par{\textbf{Experimental results:}} In Table \ref{tab:exp_results}, we have reported the results obtained from evaluating  \MIA effectiveness on the GPT-2 model and different fine-tuning algorithms. As observed in the first row of results, full fine-tuning  all parameters achieves the best utility (PPL) on all datasets with the cost of a considerable vulnerability to \MIA. In contrast, as shown in the second row of results, choosing a subset of parameters in the existing 12 layers of the model (the adapters) and fine-tuning them with \LoRA improves robustness to  \MIA considerably. As observed, \LoRA achieves a considerably lower AUC and TPR@FPR compared to full fine-tuning at the cost of a slight drop in utility (larger perplexity), thanks to the low-rank adaptation mechanism incorporating \emph{noisy} batch gradients w.r.t. the adapters. Also, from the third row, we can observe that \LoRA with a frozen $\mathbf{A}_{\ell}^0$, which is compatible with our theoretical results, enhances this robustness even further with a slight drop in utility. These results are compatible with previous observations about \LoRA in the literature \citep{liu2024precuriousinnocentpretrainedlanguage, tahseen2025}. \citet{tahseen2025}. Comparing the last two rows, shows the effect of using a larger batch size $b$ at a fixed number of steps $T=EN/b$. As discussed, using a larger batch size yields to more robustness to \MIA. 

We observed in \Cref{lem:DPofLora} that the privacy bound $\epsilon(\alpha)$ decreases linearly with rank $r$. Therefore, we expect \LoRA to show more robustness to  \MIA as $r$ decreases. This is observed in \Cref{fig:ablation_r}, where we have done an ablation study on the Pubmed dataset about the effect of $r$. As observed in the left figure, the utility of the fine-tuned model decreases as $r$ decreases. Similarly, as observed in the right figure, robustness of the fine-tuned model w.r.t.  \MIA increases as $r$ decreases. This interesting observation shows that the adaptation rank $r$ adjusts the trade-off between utility of the fine-tuned model and its robustness to \MIA.

Now, we introduce a measure to quantify the amount of noise introduced by low-rank adaptation. Let $\mathbf{W}_{\ell}^0$ denote the pretrained adapter parameter at the layer $\ell$ of the model. Also, let $\mathbf{W}_{\ell}^1$ denote the adapter parameter after one \SGD update. Similarly, let $\mathbf{A}_{\ell}^1$ and $\mathbf{B}_{\ell}^1$ denote the parameters obtained from \LoRA (with rank $r$) after one \SGD update w.r.t. $\mathbf{B}_{\ell}^0$ ($\mathbf{A}_{\ell}^0$ is frozen). In that case, the equivalent forward pass parameter will be $ (\mathbf{W}_{\ell}^0+\mathbf{B}_{\ell}^1 \mathbf{A}_{\ell}^0)$. According to our previous understandings that low-rank adaptation of $\mathbf{W}_{\ell}^0$ is equivalent to fine-tuning it with noisy batch gradients, the amount of noise introduced by low-rank adaptation to the adapter $\mathbf{W}_\ell$ after one gradient step can be measured by:

\begin{align}
    N(\mathbf{W}_{\ell}^0,r) = \big\|\mathbf{W}_{\ell}^1 - (\mathbf{W}_{\ell}^0+\mathbf{B}_{\ell}^1 \mathbf{A}_{\ell}^0)\big\|.
\end{align}

We have used the above noise measure in \cref{fig:layernorms} to confirm our theoretical findings in \Cref{lem:noise_var} that \LoRA adds more noise to the \emph{adapters} $\mathbf{W}_\ell$ that have a batch gradient with larger norm. We observe that the adapter $\mathbf{W}_6$ in the 6-th attention layer of GPT-2 model, which has one of the largest batch gradients (left figure), experiences the largest noise in one gradient step of \LoRA (right figure). Also, as seen in the right figure, the amount of noise injected to different layers decreases uniformly as the rank $r$ increases. We already observed in \cref{fig:ablation_r} that smaller $r$ yields to more robustness to privacy attacks, altogether confirming our theoretical findings in \cref{lem:noise_var} and \cref{lem:DPofLora}.

\section{Conclusion}
In this study, we uncovered a new finding about the dynamics of low-rank adaptation with \LoRA. We showed that \LoRA secretly injects random noise into the batch gradients w.r.t. \emph{adapters} non-uniformly. By quantifying the variance of this noise, we showed that the rank of adaptation controls the amount of injected noise: the smaller the rank, the larger the magnitude of the injected noise. We also showed that this randomization mechanism can be looked at as a random sketching process that sketches batch gradients w.r.t. \emph{adapters} with a Gaussian matrix during fine-tuning time. This observation enabled us to prove that \LoRA, with frozen $\{\mathbf{A}_\ell\}$ component matrices, provides \DP \emph{by design}. We also showed experimentally that low-rank adaptation indeed provides robustness against membership inference attacks (\MIA) to fine-tuning data. The rank of adaptation can effectively balance utility of the fine-tuned model and its robustness to privacy attacks.

\section{Acknowledgement}
Funding support for project activities has been partially provided by Canada CIFAR AI Chair, Facebook award, and Google award.

\bibliography{example_paper}
\bibliographystyle{icml2026}

\newpage
\appendix
\onecolumn
\begin{center}
\Large
\bf
Appendix: \emph{``\LoRA Provides Differential Privacy by Design via Random Sketching"}
\end{center}

\section{Useful results}\label{sec:useful_thms}
In this section, we mention some established results, which we will use in our proofs.




\begin{lemma}\label{lem:AABBbound}
    Let $\mathbf{A}$ and $\mathbf{B}$ be two matrices with bounded Frobenius norms: $\|\mathbf{A}\|_F \leq B_F$, $\|\mathbf{B}\|_F \leq B_F$, and bounded difference: $\|\mathbf{A} - \mathbf{B}\|_F \leq B_S$. Then:

    \begin{align}
        \|\mathbf{A}\mathbf{A}^\top - \mathbf{B}\mathbf{B}^\top\|_F \leq 2B_F B_S.
    \end{align}
\end{lemma}

\begin{proof}
    We have
    \begin{align}
        \|\mathbf{A}\mathbf{A}^\top - \mathbf{B}\mathbf{B}^\top\|_F &= \|\mathbf{A}(\mathbf{A}-\mathbf{B})^\top + (\mathbf{A}-\mathbf{B})\mathbf{B}^\top\|_F \nonumber \\
        &\leq \|\mathbf{A}(\mathbf{A}-\mathbf{B})^\top\|_F + \|(\mathbf{A}-\mathbf{B})\mathbf{B}^\top\|_F \nonumber \\
        & \leq \|\mathbf{A\|_F\|}(\mathbf{A}-\mathbf{B})^\top\|_F + \|(\mathbf{A}-\mathbf{B})\|_F\|\mathbf{B}^\top\|_F \nonumber \\
        & \leq 2B_FB_S, \nonumber
    \end{align}
    where the second last line is obtained from the fact that Frobenius norm is sub-multiplicative for matrix products.
\end{proof}

\begin{lemma}\label{lem:evofXAinvY}
    Assume $\mathbf{A}$ is an invertible matrix, and $\mathbf{X}$ is also a matrix. Then, the following bound holds for the eigen values of $\mathbf{X}^\top \mathbf{A}^{-1} \mathbf{X}$:

    \begin{align}
    \lambda_i(\mathbf{X}^\top \mathbf{A}^{-1} \mathbf{X}) \leq \frac{\|\mathbf{X}\mathbf{X}^\top\|_F}{\lambda_{min}(\mathbf{A})},
    \end{align}
    where $\lambda_{min}(\cdot)$ denotes the minimum eigen value.
\end{lemma}

\begin{proof}
    First, note that for a matrix $\mathbf{X}$, $\|\mathbf{X}^\top\|_2 = \|\mathbf{X}\|_2 = \sigma_{max}(\mathbf{X})$, where $\sigma_{max}(\cdot)$ denotes maximum singular value. Also, as $\mathbf{A}$ is invertible, $\|\mathbf{A}^{-1}\|_2 = \sigma_{max}(\mathbf{A}^{-1}) = \lambda_{max}(\mathbf{A}^{-1}) = \frac{1}{\lambda_{min}(\mathbf{A})}$, where $\lambda_{min}$ ($\lambda_{max}$) is the minimum (maximum) eigen value. Therefore, from the sub-multiplicative property of spectral norm, we have:
    \begin{align}
    \|\mathbf{X}^\top\mathbf{A}^{-1}\mathbf{Y}\|_2 &\leq \|\mathbf{X}^\top \|_2 \cdot \|\mathbf{A}^{-1}\|_2 \cdot \|\mathbf{Y}\|_2 = \sigma_{max}(\mathbf{X})\cdot \lambda_{max}(\mathbf{A})\cdot \sigma_{max}(\mathbf{Y}) = \frac{\sigma_{max}(\mathbf{X})\cdot \sigma_{max}(\mathbf{Y})}{\lambda_{min}(\mathbf{A})}.\nonumber
    \end{align}
    Now, if we set $\mathbf{X} = \mathbf{Y}$, we get:

    \begin{align}
    \|\mathbf{X}^\top\mathbf{A}^{-1}\mathbf{X}\|_2 \leq \|\mathbf{X}^\top \|_2 \cdot\|\mathbf{A}^{-1}\|_2 \cdot\|\mathbf{X}\|_2 = \frac{\sigma_{max}(\mathbf{X})\cdot \sigma_{max}(\mathbf{X})}{\lambda_{min}(\mathbf{A})} = \frac{\sigma_{max}^2(\mathbf{X})}{\lambda_{min}(\mathbf{A})} = \frac{\|\mathbf{X}\|_2^2}{\lambda_{min}(\mathbf{A})} &= \frac{\|\mathbf{X}\mathbf{X}^\top\|_2}{\lambda_{min}(\mathbf{A})} \nonumber \\
    & \leq \frac{\|\mathbf{X}\mathbf{X}^\top\|_F}{\lambda_{min}(\mathbf{A})},\nonumber
    \end{align}
    where the last equality in the first line holds as $\sigma_{max}(\mathbf{X}\mathbf{X}^\top) = \sigma_{max}^2(\mathbf{X})$.
\end{proof}

\begin{lemma}[Matrix determinant lemma, \citep{matrixalgebra}]\label{lem:MDL}
    Let $\mathbf{A} \in \mathbb{R}^{n\times n}$ be an invertible matrix. Also, let $\mathbf{U} \in \mathbb{R}^{n\times k}$ and $\mathbf{V} \in \mathbb{R}^{n\times k}$ be two matrices of rank $k$ ($k\leq n$). Then:

    \begin{align}
        \det (\mathbf{A}+\mathbf{U}\mathbf{V}^\top) = \det (\mathbf{A}) \cdot \det (\mathbb{I}_k+\mathbf{V}^\top \mathbf{A}^{-1}\mathbf{U}).
    \end{align}
\end{lemma}

\begin{lemma}[Weyl's inequality, \citep{Weyl1912DasAV}]\label{lem:weylineq}
    Let $\mathbf{A}$ and $\mathbf{B}$ be two symmetric matrices, and let $\lambda_{min}(\cdot)$ denote the minimum eigenvalue. Then:

    \begin{align}
        \lambda_{min}(\mathbf{A}+\textbf{B}) \geq \lambda_{min}(\mathbf{A}) + \lambda_{min}(\mathbf{B}).
    \end{align}
\end{lemma}

\begin{lemma}[Sylvester's determinant identity, \citep{Sylvester}]\label{lem:sylvestereq}
    Let $\mathbf{A} \in \mathbb{R}^{n \times m}$ and $\mathbf{B} \in \mathbb{R}^{m \times n}$:

    \begin{align}
        \det(\mathbb{I}_m + \mathbf{B}\mathbf{A}) = \det(\mathbb{I}_n + \mathbf{A}\mathbf{B}).
    \end{align}
\end{lemma}

\begin{theorem}[Chi-Squared distribution: \citep{AlexMood}, Section 4.3, Theorem 7]\label{thm:chisquared}
If the random variables $X_i$, $i=1, \ldots, k$, are normally and independently distributed with means $\mu_i$ and variances $\sigma_i^2$, then

\begin{align}
    U = \sum_{i=1}^k \big(\frac{X_i- \mu_i}{\sigma_i}\big)^2
\end{align}

has a chi-squared distribution with $k$ degrees of freedom: $U \sim \mathcal{X}_k^2$. Also, $\mathbb{E}[U]=k$ and $\texttt{Var}[U]=2k$. 
\end{theorem}
The theorem above states that sum of the squares of $k$ standard normal random variables is a chi-squared distribution with $k$ degrees of freedom.

\begin{lemma}[Raw moment of Chi-Squared distribution]\label{lemma:rw_moment_chisquared}
Suppose $X \sim \mathcal{X}_k^2$. Then, the $m$-th raw moment of $X$ can be found as follows;

\begin{align}
    \mathbb{E}[X^m] = \prod_{i=0}^{m-1} (k + 2i)
\end{align}
\end{lemma}

\begin{proof}
    From the definition of Chi-Squared distribution with $r$ degrees of reddom, $X$ has the following probability density function:

    \begin{align}
        f_{X}(x) = \frac{1}{2^{\frac{k}{2}} \Gamma(\frac{k}{2})} x^{\frac{k}{2}-1}  e^{-\frac{x}{2}}
    \end{align}

    Therefore, we have:

    \begin{align}
        \mathbb{E}[X^m] &= \frac{1}{2^{\frac{k}{2}} \Gamma(\frac{k}{2})} \int_0^{+\infty} x^{\frac{k}{2}+m-1} e^{-\frac{x}{2}} dx = \frac{2}{2^{\frac{k}{2}} \Gamma(\frac{k}{2})} \int_0^{+\infty} (2u)^{\frac{k}{2}+m-1} e^{-u} du \nonumber \\
        & =\frac{2^{\frac{k}{2}+m-1+1}}{2^{\frac{k}{2}} \Gamma(\frac{k}{2})} \int_0^{+\infty} u^{\frac{k}{2}+m-1} e^{-u} du = \frac{2^m}{\Gamma(\frac{k}{2})}\Gamma (\frac{k}{2}+m) = \frac{2^m \Gamma(\frac{k}{2})}{\Gamma(\frac{k}{2})}\prod_{i=0}^{m-1} (\frac{k}{2}+i) \nonumber \\
        &= \prod_{i=0}^{m-1} (k+2i).
    \end{align}
    Note that the fifth equality directly results from the property of gamma function that for $z>0$, $\Gamma(1+z) = z\Gamma(z)$.
\end{proof}

\begin{theorem}[Classical Central Limit Theorem: \citep{patrickB}, Theorem 27.1]\label{thm:CLT}
Suppose that $\{X_i\}_{i=1}^n$, is an independent sequence of random variables having the same distribution with mean $\mu$ and positive variance $\sigma^2$. Define $S_n = \sum_{i=1}^n X_i$ as their sum. Let $Z_n$ be defined by

\begin{align}
    Z_n = \frac{S_n - n \mu}{\sqrt{n} \sigma}.
\end{align}

Then, the distribution of $Z_n$ approaches standard normal distribution as $n$ approaches infinity. 

\end{theorem}
The theorem above states that $ S_n$ is approximately, or asymptotically, distributed as a normal distribution with mean $n\mu$ and variance $n \sigma^2$.

The next theorem is about the Lindeberg's condition, which is a sufficient (and under certain conditions also a necessary condition) for the Central Limit Theorem (CLT) to hold for a sequence of independent random variables $\{X_i\}_{i=1}^n$. Unlike the classical CLT stated above, which requires the sequence of random variables to have a finite variance and be both independent and identically distributed (\emph{i.i.d}), Lindeberg's CLT only requires the sequence of random variables to have finite variance, be independent and also satisfy the Lindeberg's condition. The following states the theorem.

\begin{theorem}[Lindeberg and Lyapounov Theorem: \citep{patrickB}, Theorem 27.2]\label{thm:Lindeberg_CLT}
Suppose $X_1, \ldots, X_n$ are $n$ independent random variables with $\mathbb{E}[X_i] = \mu_i$ and $\texttt{Var}[X_i]=\sigma_i^2>0$. Define $S_n = \sum_{i=1}^n X_i$ and let  $s_n^2 = \sum_{i=1}^n \sigma_i^2$. 
Also assume the following condition holds for all $\epsilon >0$:

\begin{align}
    \textit{Lindeberg's condition:}~~~\lim_{n \to \infty} \sum_{i=1}^n \frac{1}{s_n^2} \int_{|x-\mu_i|\geq \epsilon s_n} (x - \mu_i)^2 P_{X_i}(x) dx =  0.
\end{align}

where $P_{X_i}$ is the pdf of variable $X_i$. Assuming $Z_n =\frac{S_n - \sum_{i=1}^n \mu_i}{s_n}$, the distribution of $Z_n$ approaches standard normal distribution as $n$ approaches infinity. 

\end{theorem}

The theorem above states that, given that Lindeberg's condition is satisfied, $ S_n$ is approximately, or asymptotically, distributed as a normal distribution with mean $\sum_{i=1}^n \mu_i$ and variance $s_n^2$, even if the sequence of variables are not identically distributed.




\section{Analysis of the additive noise term in \Cref{eq:param_update_noise} with Kaiming-uniform initialization of $\{\mathbf{A}_\ell\}_{\ell=1}^L$} \label{app:kaiming}

Kaiming uniform initialization \citep{kaiming}, is a weight initialization technique primarily designed for neural networks employing \texttt{ReLU}-like activation functions. While it was originally derived for \texttt{ReLU}, its principles can be extended to other non-linear activations like \texttt{GELU}, which is a smoother variant of \texttt{ReLU}.
The core idea of Kaiming initialization is to set the initial weights in a way that helps maintain a consistent variance of activations throughout the network, preventing vanishing or exploding gradients during training. For uniform distribution, the weights are drawn from a uniform distribution within a specific range $\mathcal{U}(-b, b)$, where $b=\sqrt{6/m}$. The value of $b$ is found for each layer such that the variance of each subsequent pre-activation layer is equal, e.g., in \Cref{fig:lora}, the variance of each element in $y$ is equal to that of each element in the pre-activation of input $x$. However, this initialization does not minimize the magnitude of the \LoRA injected noise in 
\Cref{eq:param_update_noise}. We have the following lemma.

\begin{restatable}{lemma}{noisevarkaiming}\label{lem:noise_var_kaiming}
Let $\mathbf{A}_{\ell} \in \mathbb{R}^{r\times m}$ be a matrix with \emph{i.i.d} entries sampled from $\mathcal{U}(-b, b)$. Given a fixed $\mathbf{q}\in \mathbb{R}^{1 \times m}$,  the distribution of the $\mathbf{q}\cdot (\mathbf{A}_{\ell}^\top \mathbf{A}_{\ell}-\mathbb{I}_m) \in \mathbb{R}^{1 \times m}$ approaches the $m$-dimensional Gaussian distribution $\mathcal{N}\big((r\sigma_{\mathbf{A}}^2 - 1)~\mathbf{q}, ~r\sigma_{\mathbf{A}}^4 (\|\mathbf{q}\|_2^2 ~ \mathbb{I}_m + \mathbf{D})\big)$, as $m$ increases. Also, $\sigma_{\mathbf{A}}^2=b^2/3$ is the variance of each element in $\mathbf{A}_{\ell}$ and $\mathbf{D}\in \mathbb{R}^{m\times m}$ is a diagonal matrix with $\mathbf{D}_{i,i}=-\mathbf{q}_i^2/5$. 
\end{restatable}

This result can be extended to matrices multiplication, as in \cref{eq:param_update_noise}: for $A \in \mathbb{R}^{r\times m}$ sampled from $\mathcal{U}(-b, b)$ and $\mathbf{Q} \in \mathbb{R}^{n\times m}$, as $m$ grows, the product $\mathbf{G} = \mathbf{Q} \cdot (\mathbf{A}_{\ell}^\top \mathbf{A}_{\ell}-\mathbb{I}_m) \in \mathbb{R}^{n\times m}$ approaches a Gaussian distribution, where $\mathbf{G}_{i,:}$ ($i$-th row of $\mathbf{G}$, $1\leq i\leq n$) has distribution $\mathcal{N}\big((r\sigma_{\mathbf{A}}^2-1)[\mathbf{Q}]_{i,:}, ~ r\sigma_{\mathbf{A}}^4 (\|[\mathbf{Q}]_{i,:}\|_2^2 ~ \mathbb{I}_m + \mathbf{D} \big)$, where $[\mathbf{Q}]_{i,:}$ is the $i$-th row of $\mathbf{Q}$ and $\mathbf{D}\in \mathbb{R}^{m\times m}$ is the diagonal matrix built based on it: $\mathbf{D}_{j,j} = -\mathbf{Q}_{i,j}^2/5$.

Replacing $\mathbf{Q}$ above with $\nabla_{\mathbf{W}_{\ell}^t} \mathcal{L}^t$, we get to the noise summand in \Cref{eq:param_update_noise}: the product $\mathbf{G} = \nabla_{\mathbf{W}_{\ell}^t} \mathcal{L}^t \cdot (\mathbf{A}_{\ell}^\top \mathbf{A}_{\ell} - \mathbb{I}_m)\in \mathbb{R}^{n\times m}$ approaches a Gaussian distribution, where $\mathbf{G}_{i,:}$ ($i$-th row of $\mathbf{G}$, $1\leq i\leq n$) has distribution $\mathcal{N}\big((r\sigma_{\mathbf{A}}^2-1)[\nabla_{\mathbf{W}_{\ell}^t} \mathcal{L}^t]_{i,:}, ~ r\sigma_{\mathbf{A}}^4 (\|[\nabla_{\mathbf{W}_{\ell}^t} \mathcal{L}^t]_{i,:}\|_2^2 ~ \mathbb{I}_m + \mathbf{D}) \big)$, where $[\nabla_{\mathbf{W}_{\ell}^t} \mathcal{L}^t]_{i,:}$ is the $i$-th row of $\nabla_{\mathbf{W}_{\ell}^t} \mathcal{L}^t$ and $\mathbf{D}\in \mathbb{R}^{m\times m}$ is the diagonal matrix built based on it: $\mathbf{D}_{j,j} = -[\nabla_{\mathbf{W}_{\ell}^t} \mathcal{L}^t]_{i,j}^2/5$. 
With Kaiming-uniform initialization $b=\sqrt{6/m}$, $\sigma_{\mathbf{A}}^2 = b^2/3 = 2/m$. Therefore, $r\sigma_{\mathbf{A}}^2 - 1 = 2r/m - 1$. In other words, Kaiming-uniform initialization leads to the injection of a \textbf{\emph{biased}} noise into the batch gradients: at each step $t$, distribution of $\nabla_{\mathbf{W}_{\ell}^t} \mathcal{L}^t\cdot (\mathbf{A}_{\ell}^\top \mathbf{A}_{\ell}-\mathbb{I}_m)$ is centered around $(\frac{2r}{m}-1) \nabla_{\mathbf{W}_{\ell}^t} \mathcal{L}^t \approx - \nabla_{\mathbf{W}_{\ell}^t} \mathcal{L}^t$ with a small variance. This means that the noise term in \Cref{eq:param_update_noise} is unbiased and almost cancels the batch gradient term at every time step $t$ and slows down the fine-tuning procedure leading to low utility.

\section{Further analysis of the additive noise term in \Cref{eq:param_update_noise} with Gaussian analysis of $\{\mathbf{A}_\ell\}_{\ell=1}^L$}\label{app:complete_analysis}

Although \cref{lem:noise_var} was proved when $m$ approaches infinity, in practical scenarios $m$ is limited (the input dimension of the adaptation layer). Hence, the distribution of the injected noise is not a pure Gaussian.

\subsection{Bounding the distance between the additive noise in \Cref{eq:param_update_noise} to Gaussian noise}\label{app:gaussian_deviation}

In this section, we provide more interesting details about the Gaussian initialization of $\mathbf{A}_{\ell}$, as it is a common practice. Having proved \cref{lem:noise_var} when $m$ approaches infinity, yet we need to quantify the distance between the distribution of $\mathbf{q}\cdot (\mathbf{A}_{\ell}^\top \mathbf{A}_{\ell}-\mathbb{I}_m) \in \mathbb{R}^{1 \times m}$ (for a given $\mathbf{q}\in \mathbb{R}^{1 \times m}$) and the Gaussian distribution with the same variance for limited values of $m$. In this section, we assume $\sigma_{\mathbf{A}}^2 = 1/r$, and we derive a closed form upper-bound for the total variation distance between the distribution of $\mathbf{q}\cdot (\mathbf{A}_{\ell}^\top \mathbf{A}_{\ell}-\mathbb{I}_m) \in \mathbb{R}^{1 \times m}$ and the Gaussian distribution $\mathcal{N}\big(\mathbf{0}, ~(\|\mathbf{q}\|_2^2~\mathbb{I}_m + \mathbf{D})/r \big)$. 

Suppose $X_1, \ldots, X_n$ are $n$ independent random variables that are not necessarily \emph{i.i.d}, and $\mathbb{E}[X_i] = 0$ (mean) and $\texttt{Var}[X_i]=\sigma_i^2>0$ (variance). Define $S_n = \sum_{i=1}^n X_i$ and let  $s_n^2 = \sum_{i=1}^n \sigma_i^2$. Assuming $Z_n =S_n/s_n$, and having Lindeberg's condition satisfied (see \cref{thm:CLT} and \cref{thm:Lindeberg_CLT} in the appendix), the normalized sum $Z_n$ has standard normal distribution in a weak sense for a bounded $n$. More precisely, the closeness of the cumulative distribution function (CDF) $F_n(x)=\texttt{Pr}\{Z_n\leq x\}$ to the standard normal CDF has been studied in terms of the Lyapounov ratios:

\begin{align}\label{eq:lyapounov_ratios}
    L_t = \frac{\sum_{i=1}^n \mathbb{E}[|X_i|^t]}{s_n^{t}}.
\end{align}

Particularly, if all $\{X_i\}_{i=1}^n$ are \emph{i.i.d} and have a finite third absolute moment $\mathbb{E}[|X_i|^3]$, the classical Berry-Esseen theorem \citep{Esseen1945FourierAO, feller1971, Petrov1975} bounds the Kolmogrov distance between $F_n(x)$ and $\Phi(x)$: 

\begin{align}
    \sup_x |F_n(x) - \Phi(x)| \leq CL_3,
\end{align} 

where $C$ is an absolute constant. In the more general case, when $\{X_i\}_{i=1}^n$ are independent (and not necessarily \emph{i.i.d}), the number of summand variables $n$ implicitly affects the value of $L_3$. Yet, we can bound the difference between $F_n(x)$ and $\Phi(x)$ in terms of generally stronger distances of total variation or entropic distance \citep{Bobkov2011}. To this end, let $D(X_i)$ denote the KL divergence between $X_i$ and a normal variable from $\mathcal{N}(0, \sigma_i^2)$, i.e., the KL divergence between $X_i$ and a normal variable with the same variance. Therefore, we have the following theorem about the distance between $F_n$ and $\Phi$:

\begin{theorem}[\citet{Bobkov2011}, theorem 1.1]
\label{thm:tv_distance}
Assume that the independent random variables $X_1, \ldots, X_n$ have finite third absolute moments, and that $D(X_i)\leq D$, where $D$ is a non-negative number. Then, 

\begin{align}\label{eq:bobkov_bound}
    \|F_n(x)-\Phi(x)\|_{TV} \leq C_D L_3,
\end{align}
where the constant $C_D$ depends on $D$ only and $\|F_n(x)-\Phi(x)\|_{TV} = \sup_A \big|\int_A dF_n- \int_A d\Phi \big|$ is the total variation distance between $F_n$ and $\Phi$.

\end{theorem}

Having the theorem above, we can now derive a Berry-Esseen type bound for the total variation distance between each element of $\mathbf{q}\cdot (\mathbf{A}_{\ell}^\top \mathbf{A}_{\ell}-\mathbb{I}_m) \in \mathbb{R}^{1 \times m}$ in \cref{lem:noise_var} and the normal law $\mathcal{N}(0, \|\mathbf{q}\|^2/r)$: we need to find the third Lyapounov ratio for the summands contributing to each of the $m$ elements in $\mathbf{q}\cdot (\mathbf{A}_{\ell}^\top \mathbf{A}_{\ell}-\mathbb{I}_m) \in \mathbb{R}^{1 \times m}$. To this end, we state and prove the following lemma:

\begin{restatable}{lemma}{TVdist}\label{lem:total_var}
Let $A \in \mathbb{R}^{r\times m}$ be a matrix with \emph{i.i.d} entries sampled from $\mathcal{N}(0, \frac{1}{r})$. Given a fixed $\mathbf{q}\in \mathbb{R}^{1 \times m}$, let $u = \mathbf{q}\cdot (\mathbf{A}_{\ell}^\top \mathbf{A}_{\ell}-\mathbb{I}_m) \in \mathbb{R}^{1 \times m}$. Let $u_i$ be the $i$-th element of $u$ and $Q_m(x) = \texttt{Pr}\{u_i\leq x\}$, the CDF of $u_i$. Also, let $\Phi(x)$ be the CDF of $z\sim \mathcal{N}(0, \frac{\|\mathbf{q}\|^2 + \mathbf{q}_i^2}{r})$. Then:

\begin{align}
    \|Q_m(x)-\Phi(x)\|_{TV} \in \mathcal{O}\bigg(\frac{1}{\sqrt{r}}\bigg).
\end{align}
\end{restatable}

The above result shows that the distribution of $\mathbf{q}\cdot (\mathbf{A}_{\ell}^\top \mathbf{A}_{\ell}-\mathbb{I}_m) \in \mathbb{R}^{1 \times m}$ gets closer and closer to $\mathcal{N}\big(\mathbf{0}, ~(\|\mathbf{q}\|_2^2~\mathbb{I}_m + \mathbf{D})/r \big)$ as $m$ and $r$ increase. We can now benefit from the useful coupling characterization of the total variation distance to establish a more understandable interpretation of the above result.

\paragraph{The coupling characterization of the total variation distance.}
For two distributions $P$ and $Q$, a pair of random variables $(X, Y)$, which are defined on the same probability space, is called a coupling for $P$ and $Q$ if $X\sim P$ and $Y\sim Q$ \citep{Levin2008MarkovCA, abbas}. A very useful property of total variation distance is the coupling characterization \citep{Levin2008MarkovCA}: 

$\|P-Q\|_{TV}\leq t$ if and only if there exists a coupling $(X,Y)$ for them such that $\texttt{Pr}\{X\neq Y\}\leq t$. 

We can use the coupling characterization to prove the following lemma from \cref{lem:total_var}.

\begin{restatable}{lemma}{TVresult}\label{lem:coupling}
Let $A \in \mathbb{R}^{r\times m}$ be a matrix with \emph{i.i.d} entries sampled from $\mathcal{N}(0, \frac{1}{r})$. Given a fixed $\mathbf{q}\in \mathbb{R}^{1 \times m}$, let $u = \mathbf{q}\cdot (\mathbf{A}_{\ell}^\top \mathbf{A}_{\ell}-\mathbb{I}_m) \in \mathbb{R}^{1 \times m}$. Let $u_i$ be the $i$-th element of $u$. There exists a coupling $(u_i, z)$, where $z\sim \mathcal{N}(0, \frac{\|\mathbf{q}\|^2 + \mathbf{q}_i^2}{r})$ and

\begin{align}
    \texttt{Pr}\{u_i \neq z\} \in  \mathcal{O}\bigg(\frac{1}{\sqrt{r}}\bigg),
\end{align}
\end{restatable}

with a small multiplicative factor (see the discussion after \Cref{eq:tv} in the appendix). The lemma above means that the $i$-th element  of $\mathbf{q}\cdot (\mathbf{A}_{\ell}^\top \mathbf{A}_{\ell}-\mathbb{I}_m)$ follows a mixture of distributions: $\mathcal{N}(0, \frac{\|\mathbf{q}\|^2 + \mathbf{q}_i^2}{r})$ with weight $w_g$ and another distribution $M$ - which we dont know - with weight $(1 - w_g) \in \mathcal{O}\big(\frac{1}{\sqrt{r}}\big)$. 
The larger $r$, the closer the mixture distribution gets to pure Gaussian distribution $\mathcal{N}(0, \frac{\|\mathbf{q}\|^2 + \mathbf{q}_i^2}{r})$. Having established the results above, we can now make a clear comparison between the dynamics of low-rank adaptation of $\{\mathbf{W}_{\ell}^t\}_{\ell=1}^L$ and the dynamics of \DPSGD (with non-uniform noise addition) w.r.t. $\{\mathbf{W}_{\ell}^t\}_{\ell=1}^L$.

\section{Proofs}

\subsection{Proof of \Cref{lem:DPofRS}}\label{app:proofofdpofrs}
\DPofRS*

    Our proof is built upon the following lemma for the R\'enyi divergence between two multi-variate Gaussian distributions:

    \begin{lemma}[\citet{GIL2013124}]\label{lem:gaussiansrenyi}
    Let $\mathcal{M}_1 \sim \mathcal{N}(\mathbf{\mu}_1, \mathbf{\Sigma}_1$ and $\mathcal{M}_2 \sim \mathcal{N}(\mathbf{\mu}_2, \mathbf{\Sigma}_2)$. Then, the R\'enyi divergence between the two can be computed as:

    \begin{align}\label{eq:alpharenyi}
        D_{\alpha}(\mathcal{M}_1||\mathcal{M}_2) = \frac{\alpha}{2}(\mathbf{\mu}_1 - \mathbf{\mu}_2)^\top \big(\mathbf{\Sigma}_1 + \alpha (\mathbf{\Sigma}_2 - \mathbf{\Sigma}_1)\big)^{-1}(\mathbf{\mu}_1 - \mathbf{\mu}_2) -\frac{1}{2(\alpha -1)} \log\bigg(\frac{\det(\mathbf{\Sigma}_1 + \alpha (\mathbf{\Sigma}_2 - \mathbf{\Sigma_1}))}{(\det(\mathbf{\Sigma}_1))^{1-\alpha} (\det(\mathbf{\Sigma_2}))^\alpha}\bigg)
    \end{align}
    for all $\alpha$ such that $\alpha \Sigma_1^{-1} + (1-\alpha)\Sigma_2^{-1} \succ 0$.

    \end{lemma}

\begin{proof}
The $r$ columns of $\mathcal{M}(\mathbf{Q}(\mathcal{D}))$ are \emph{i.i.d} multi-variate ($n$-dimensional) Gaussian variables with mean zero and covariance matrix $\mathbf{\Sigma}_1 = \sigma_A^2 \mathbf{Q}(\mathcal{D})\mathbf{Q}(\mathcal{D})^\top + \sigma_g^2 \mathbb{I}_n$. Similarly, the $r$ columns of $\mathcal{M}(\mathbf{Q}(\mathcal{D}'))$ are \emph{i.i.d} multi-variate ($n$-dimensional) Gaussian variables with mean zero and covariance matrix $\mathbf{\Sigma}_2 = \sigma_A^2 \mathbf{Q}(\mathcal{D}')\mathbf{Q}(\mathcal{D}')^\top + \sigma_g^2 \mathbb{I}_n$. Therefore, $\mathbf{\Sigma}_1, \mathbf{\Sigma}_2 \in \mathbb{R}^{n\times n}$ are both positive semi-definite matrices. Consequently, if we concatenate the $r$ columns of $\mathcal{M}(\mathbf{Q}(\mathcal{D}))$ and $\mathcal{M}(\mathbf{Q}(\mathcal{D}'))$ into two $nr$-dimensional vectors, they will have the following distributions;

\begin{align*}
    \text{vec}{(\mathcal{M}(\mathbf{Q}(\mathcal{D})))} \sim \mathcal{N}(\mathbf{0}, \mathbb{I}_r \otimes\mathbf{\Sigma}_1), \\
    \text{vec}{(\mathcal{M}(\mathbf{Q}(\mathcal{D}')))} \sim \mathcal{N}(\mathbf{0}, \mathbb{I}_r \otimes\mathbf{\Sigma}_2),
\end{align*}
where $\otimes$ denotes the Kronecker product between two matrices. Therefore, based on \Cref{lem:gaussiansrenyi}, the R\'enyi divergence between $\mathcal{M}(\mathbf{Q}(\mathcal{D}))$ and $\mathcal{M}(\mathbf{Q}(\mathcal{D}'))$ can be written as:

\begin{align}\label{eq:renyiddprime}
    D_{\alpha}(\mathcal{M}(\mathbf{Q}(\mathcal{D}))||\mathcal{M}(\mathbf{Q}(\mathcal{D}'))) = \frac{-r}{2(\alpha - 1)} \log \bigg(\frac{\det (\mathbf{\Sigma_1} + \alpha \Delta \mathbf{\Sigma})}{(\det (\mathbf{\Sigma_1}))^{1-\alpha} (\det (\mathbf{\Sigma_2}))^\alpha}\bigg).
\end{align}

The matrix $\Delta \mathbf{\Sigma} \in \mathbb{R}^{n\times n}$ is symmetric with $n$ real eigenvalues. Due to closeness of $ \mathbf{\Sigma}_1$ and $ \mathbf{\Sigma}_2$, without loss of generality, we can assume that $\Delta \mathbf{\Sigma}$ has rank $\tau \leq n$. So we can write $\Delta \mathbf{\Sigma} = \mathbf{V} \mathbf{\Lambda}\mathbf{V}^\top$, where $\mathbf{V} \in \mathbb{R}^{n \times n}$ contains eigenvectors of $\Delta \mathbf{\Sigma}$ and is orthonormal ($\mathbf{V}^\top = \mathbf{V}^{-1}$). Also, the diagonal matrix $\mathbf{\Lambda} \in \mathbb{R}^{n \times n}$ contains the real eigen values of $\Delta \mathbf{\Sigma}$ (which could be positive or negative). Let us, without loss of generality, assume that the first $\tau$ diagonal elements of $\Lambda$ are non-zero and the rest are zero (remember we assumed previously that the rank of $\Delta \mathbf{\Sigma}$ is $\tau$) and the first $\tau^+$ diagonal elements are positive and the rest of the $\tau^-=\tau-\tau^+$ elements are negative. Therfore, we can split $\mathbf{\Lambda}\in \mathbb{R}^{n \times n}$ as $\mathbf{\Lambda}= \mathbf{\Lambda}^+ - \mathbf{\Lambda}^-$, where $\mathbf{\Lambda}^+$ includes only the positive eigen values and $\mathbf{\Lambda}^-$ includes only the \emph{absolute value} of negative eigen values, as shown below:

\begin{align}
    \Lambda = \Lambda^+ - \Lambda^- = 
  \begin{bmatrix}
    \color{blue}{\lambda_1} & & & & & & & & & \\
     & \color{blue}{\ddots}& & & & &0 & & & \\
     & & \color{blue}{\lambda_{\tau^+}}& & & & & & & \\
     & & & \hspace{-0.6cm}\color{brown}{0}& & & & & & \\
     & & & & \hspace{-0.2cm}\color{brown}{\ddots}& & & & & \\
     & & & & & \color{brown}{0}& & & & \\
     & & & & & & 0& & & \\
     & & 0& & & & & \ddots& & \\
     & & & & & & & & & \hspace{-0.5cm}0\\
  \end{bmatrix} -
  \begin{bmatrix}
    \color{blue}{0} & & & & & & & & & \\
     & \color{blue}{\ddots}& & & & &0 & & & \\
     & & \color{blue}{0}& & & & & & & \\
     & & & \hspace{-0.2cm} \color{brown}{\lambda_{\tau^++1}}& & & & & & \\
     & & & & \hspace{-0.4cm}\color{brown}{\ddots}& & & & & \\
     & & & & & \color{brown}{\lambda_{\tau}}& & & & \\
     & & & & & & \hspace{-0.2cm}0& & & \\
     & & 0& & & & & \ddots& & \\
     & & & & & & & & & \hspace{-0.4cm}0\\
  \end{bmatrix}
\end{align}

Therefore, we can rewrite $\Delta \mathbf{\Sigma}$ as:
\begin{align}
    \mathbf{\Sigma_1} + \Delta \mathbf{\Sigma} = \mathbf{\Sigma_1} + \mathbf{V}\mathbf{\Lambda}\mathbf{V}^\top & = \mathbf{\Sigma_1} + \mathbf{V}\mathbf{\Lambda^{+}}\mathbf{V}^\top - \mathbf{V}\mathbf{\Lambda^{-}}\mathbf{V}^\top = \mathbf{\Sigma_1} + \Delta \mathbf{\Sigma}^+ - \Delta \mathbf{\Sigma}^-
    \nonumber \\
    &= \mathbf{\Sigma_1} + (\mathbf{V}\mathbf{\Lambda^{+^\frac{1}{2}}})(\mathbf{V}\mathbf{\Lambda^{+^\frac{1}{2}}})^\top - (\mathbf{V}\mathbf{\Lambda^{-^\frac{1}{2}}})(\mathbf{V}\mathbf{\Lambda^{-^\frac{1}{2}}})^\top,
\end{align}

where $\Delta \mathbf{\Sigma}^+ := \mathbf{V}\mathbf{\Lambda}^+\mathbf{V}^\top$ and $\Delta \mathbf{\Sigma}^- := \mathbf{V}\mathbf{\Lambda}^-\mathbf{V}^\top$ are both positive semi-definite matrices. Since $\Delta \mathbf{\Sigma}^+$ and $\Delta \mathbf{\Sigma}^-$ are formed from orthonormal projections, we have: $\|\Delta \mathbf{\Sigma}\|_F^2 = \|\Delta \mathbf{\Sigma}^+\|_F^2 + \|\Delta \mathbf{\Sigma}^-\|_F^2$. Therefore

\begin{align}\label{eq:norm_ineq}
    \|\Delta \mathbf{\Sigma}^+\|_F^2 \leq \|\Delta \mathbf{\Sigma}\|_F^2, \nonumber \\
    \|\Delta \mathbf{\Sigma}^-\|_F^2 \leq \|\Delta \mathbf{\Sigma}\|_F^2
\end{align}

We now find the terms appearing in \Cref{eq:renyiddprime}, as follows:

\begin{align}
    \det(\mathbf{\Sigma_1} + \Delta \mathbf{\Sigma}) = \det \bigg(\mathbf{\Sigma_1} + (\mathbf{V}\mathbf{\Lambda^{+^\frac{1}{2}}})(\mathbf{V}\mathbf{\Lambda^{+^\frac{1}{2}}})^\top - (\mathbf{V}\mathbf{\Lambda^{-^\frac{1}{2}}})(\mathbf{V}\mathbf{\Lambda^{-^\frac{1}{2}}})^\top \bigg).
\end{align}

We apply the \Cref{lem:MDL} to the above determinant iteratively to rewrite it. First, we have:

\begin{align}\label{eq:plus_delta_plus}
    \det(\mathbf{\Sigma_1 + \Delta \mathbf{\Sigma^+}}) = \det (\mathbf{\Sigma_1} + (\mathbf{V}\mathbf{\Lambda^{+^\frac{1}{2}}})(\mathbf{V}\mathbf{\Lambda^{+^\frac{1}{2}}})^\top) \stackrel{\ref{lem:MDL}}{=} \det (\mathbf{\Sigma}_1)\cdot \det \bigg(\mathbb{I}_n + (\mathbf{V}\mathbf{\Lambda^{+^\frac{1}{2}}})^\top\mathbf{\Sigma}_1^{-1}(\mathbf{V}\mathbf{\Lambda^{+^\frac{1}{2}}})\bigg).
\end{align}

Second, we re-apply the lemma in the following:

\begin{align}\label{eq:plus_delta}
    &\det(\mathbf{\Sigma_1 + \Delta \mathbf{\Sigma}}) = \det \bigg( \mathbf{\Sigma_1} + (\mathbf{V}\mathbf{\Lambda^{+^\frac{1}{2}}})(\mathbf{V}\mathbf{\Lambda^{+^\frac{1}{2}}})^\top - (\mathbf{V}\mathbf{\Lambda^{-^\frac{1}{2}}})(\mathbf{V}\mathbf{\Lambda^{-^\frac{1}{2}}})^\top\bigg) \nonumber \\
    & \stackrel{\ref{lem:MDL}}{=} \det \bigg(\mathbf{\Sigma_1} + (\mathbf{V}\mathbf{\Lambda^{+^\frac{1}{2}}})(\mathbf{V}\mathbf{\Lambda^{+^\frac{1}{2}}})^\top\bigg) \cdot \det \bigg ( \mathbb{I}_n - (\mathbf{V}\mathbf{\Lambda^{-^\frac{1}{2}}})^\top\bigg(\mathbf{\Sigma_1} + (\mathbf{V}\mathbf{\Lambda^{+^\frac{1}{2}}})(\mathbf{V}\mathbf{\Lambda^{+^\frac{1}{2}}})^\top\bigg)^{-1}(\mathbf{V}\mathbf{\Lambda^{-^\frac{1}{2}}})\bigg)\nonumber \\
    &\stackrel{eq. ~\ref{eq:plus_delta_plus}}{=} \det (\mathbf{\Sigma}_1)\cdot \det \bigg(\mathbb{I}_n + \underbrace{(\mathbf{V}\mathbf{\Lambda^{+^\frac{1}{2}}})^\top\mathbf{\Sigma}_1^{-1}(\mathbf{V}\mathbf{\Lambda^{+^\frac{1}{2}}})}_{\mathcal{Q}} \bigg)\cdot \det \bigg ( \mathbb{I}_n - \underbrace{(\mathbf{V}\mathbf{\Lambda^{-^\frac{1}{2}}})^\top\bigg(\mathbf{\Sigma_1} + (\mathbf{V}\mathbf{\Lambda^{+^\frac{1}{2}}})(\mathbf{V}\mathbf{\Lambda^{+^\frac{1}{2}}})^\top\bigg)^{-1}(\mathbf{V}\mathbf{\Lambda^{-^\frac{1}{2}}})}_{\mathcal{H}}\bigg) \nonumber \\
    &= \det (\mathbf{\Sigma}_1)\cdot \det \big(\mathbb{I}_n + \mathcal{Q} \big)\cdot \det \big( \mathbb{I}_n - \mathcal{H}\big) = \det(\mathbf{\Sigma}_1) \cdot \prod_{i=1}^{\tau^+} (1+\lambda_i(\mathcal{Q}))\cdot \prod_{i=1}^{\tau^-} (1-\lambda_i(\mathcal{H})),
\end{align}

where we have defined the matrices $\mathcal{Q}$ and $\mathcal{H}$ in the second last line. Note that both $\mathbf{\Sigma}_1$ and $\mathbf{\Sigma}_1 + \Delta \mathbf{\Sigma}^+$ are positive-semi-definite and invertible matrices. Also, $\mathbf{V}\mathbf{\Lambda^{+^\frac{1}{2}}}$ and $\mathbf{V}\mathbf{\Lambda^{-^\frac{1}{2}}}$ have ranks $\tau^+$ and $\tau^-$, respectively. Therefore, $\mathcal{Q}$ and $\mathcal{H}$ are positive semi-definite and have the same ranks $\tau^+$ and $\tau^-$, respectively. Hence, the last equality holds.

In a similar way, we can re-write the other term appearing in \Cref{eq:alpharenyi}, as follows:

\begin{align}\label{eq:plus_alpha_delta}
    &\det(\mathbf{\Sigma_1 + \alpha \Delta \mathbf{\Sigma}}) = \det \bigg( \mathbf{\Sigma_1} + \alpha(\mathbf{V}\mathbf{\Lambda^{+^\frac{1}{2}}})(\mathbf{V}\mathbf{\Lambda^{+^\frac{1}{2}}})^\top - \alpha(\mathbf{V}\mathbf{\Lambda^{-^\frac{1}{2}}})(\mathbf{V}\mathbf{\Lambda^{-^\frac{1}{2}}})^\top\bigg) \nonumber \\
    & \stackrel{\ref{lem:MDL}}{=} \det \bigg(\mathbf{\Sigma_1} + \alpha(\mathbf{V}\mathbf{\Lambda^{+^\frac{1}{2}}})(\mathbf{V}\mathbf{\Lambda^{+^\frac{1}{2}}})^\top\bigg) \cdot \det \bigg ( \mathbb{I}_n - \alpha(\mathbf{V}\mathbf{\Lambda^{-^\frac{1}{2}}})^\top \bigg(\mathbf{\Sigma_1} + \alpha(\mathbf{V}\mathbf{\Lambda^{+^\frac{1}{2}}})(\mathbf{V}\mathbf{\Lambda^{+^\frac{1}{2}}})^\top\bigg)^{-1}(\mathbf{V}\mathbf{\Lambda^{-^\frac{1}{2}}})\bigg)\nonumber \\
    &= \det (\mathbf{\Sigma}_1)\cdot \det \bigg(\mathbb{I}_n + \alpha(\mathbf{V}\mathbf{\Lambda^{+^\frac{1}{2}}})^\top\mathbf{\Sigma}_1^{-1}(\mathbf{V}\mathbf{\Lambda^{+^\frac{1}{2}}}) \bigg)\cdot \det \bigg ( \mathbb{I}_n - \alpha\underbrace{(\mathbf{V}\mathbf{\Lambda^{-^\frac{1}{2}}})^\top\bigg(\mathbf{\Sigma_1} + \alpha(\mathbf{V}\mathbf{\Lambda^{+^\frac{1}{2}}})(\mathbf{V}\mathbf{\Lambda^{+^\frac{1}{2}}})^\top\bigg)^{-1}(\mathbf{V}\mathbf{\Lambda^{-^\frac{1}{2}}})}_{\mathcal{J}}\bigg)\nonumber \\
    &= \det (\mathbf{\Sigma}_1)\cdot \det \big(\mathbb{I}_n + \alpha \mathcal{Q} \big)\cdot \det \big( \mathbb{I}_n - \alpha \mathcal{J}\big) = \det(\mathbf{\Sigma}_1) \cdot \prod_{i=1}^{\tau^+} (1+\alpha \lambda_i(\mathcal{Q}))\cdot \prod_{i=1}^{\tau^-} (1-\alpha \lambda_i(\mathcal{J})),
\end{align}
where we have defined the matrix $\mathcal{J}$ in the second last line above. Following the same reasoning for $\mathcal{H}$, $\mathcal{J}$ has also rank $\tau^-$. Hence, the last equality holds. Also, as $(\mathbf{V}\mathbf{\Lambda^{+^\frac{1}{2}}})(\mathbf{V}\mathbf{\Lambda^{+^\frac{1}{2}}})^\top$ is a positive semi-definite matrix and $\alpha > 1$, we have:

\begin{align}\label{eq:egenvalue_comparison}
    0 < \lambda_i(\mathcal{H}) \leq \lambda_i(\mathcal{J})
\end{align}

We can now rewrite \Cref{eq:renyiddprime} as:

\begin{align}\label{eq:renyiddprime_simplified}
D_{\alpha}(\mathcal{M}(\mathbf{Q}(\mathcal{D}))||\mathcal{M}(\mathbf{Q}(\mathcal{D}'))) &= \frac{-r}{2(\alpha - 1)} \log \bigg(\frac{\det (\mathbf{\Sigma_1} + \alpha \Delta \mathbf{\Sigma})}{(\det (\mathbf{\Sigma_1}))^{1-\alpha} (\det (\mathbf{\Sigma_1 + \Delta \mathbf{\Sigma}}))^\alpha}\bigg) \nonumber \\
& = \frac{-r}{2(\alpha - 1)} \log \bigg( \frac{\prod_{i=1}^{\tau^+} (1+\alpha \lambda_i(\mathcal{Q}))\cdot \prod_{i=1}^{\tau^-} (1- \alpha \lambda_i(\mathcal{J}))}{\prod_{i=1}^{\tau^+} (1+ \lambda_i(\mathcal{Q}))^{\alpha}\cdot \prod_{i=1}^{\tau^-} (1-\lambda_i(\mathcal{H}))^{\alpha}}\bigg)\nonumber \\
& = \frac{r}{2(\alpha - 1)} \bigg ( \sum_{i=1}^{\tau^+} \log \frac{ (1+ \lambda_i(\mathcal{Q}))^{\alpha}}{(1+\alpha \lambda_i(\mathcal{Q}))} + \sum_{i=1}^{\tau^-} \log \frac{(1- \lambda_i(\mathcal{H}))^{\alpha}}{(1- \alpha \lambda_i(\mathcal{J}))}\bigg) \nonumber \\
& \leq \frac{r}{2(\alpha - 1)} \bigg ( \sum_{i=1}^{\tau^+} \log \frac{ (1- \lambda_i(\mathcal{Q}))^{\alpha}}{(1-\alpha \lambda_i(\mathcal{Q}))} + \sum_{i=1}^{\tau^-} \log \frac{(1- \lambda_i(\mathcal{H}))^{\alpha}}{(1- \alpha \lambda_i(\mathcal{J}))}\bigg),   
\end{align}

We now bound the eigen values $\lambda_i(\mathcal{Q})$, $\lambda_i(\mathcal{J})$ and $\lambda_i(\mathcal{H})$, \emph{which are all nonnegative}. We have:

\begin{align}\label{eq:eigenvaluesofH1}
    \lambda_i(\mathcal{H}) &= \lambda_i\bigg((\mathbf{V}\mathbf{\Lambda^{-^\frac{1}{2}}})^\top\bigg(\mathbf{\Sigma_1} + (\mathbf{V}\mathbf{\Lambda^{+^\frac{1}{2}}})(\mathbf{V}\mathbf{\Lambda^{+^\frac{1}{2}}})^\top\bigg)^{-1}(\mathbf{V}\mathbf{\Lambda^{-^\frac{1}{2}}})\bigg) \stackrel{\ref{lem:evofXAinvY}}{\leq}\frac{\|(\mathbf{V}\mathbf{\Lambda^{-^\frac{1}{2}}})(\mathbf{V}\mathbf{\Lambda^{-^\frac{1}{2}}})^\top\|_F}{\lambda_{min}\big(\mathbf{\Sigma_1} + (\mathbf{V}\mathbf{\Lambda^{+^\frac{1}{2}}})(\mathbf{V}\mathbf{\Lambda^{+^\frac{1}{2}}})^\top\big)} \nonumber \\
    & = \frac{\|\Delta \mathbf{\Sigma}^-\|_F}{\lambda_{min}\big(\mathbf{\Sigma_1} + \Delta \mathbf{\Sigma}^+\big)} \stackrel{\ref{lem:weylineq}}{\leq} \frac{\|\Delta \mathbf{\Sigma}^-\|_F}{\lambda_{min}\big(\mathbf{\Sigma_1}\big)} \stackrel{\ref{eq:norm_ineq}}{\leq} \frac{\|\Delta \mathbf{\Sigma}\|_F}{\lambda_{min}\big(\mathbf{\Sigma_1}\big)} = \frac{\|\sigma_A^2 \mathbf{Q}(\mathcal{D})\mathbf{Q}(\mathcal{D})^\top - \sigma_A^2 \mathbf{Q}(\mathcal{D}')\mathbf{Q}(\mathcal{D}')^\top\|_F}{\lambda_{min}\big(\mathbf{\Sigma_1}\big)} \nonumber \\
    & \leq \frac{\sigma_A^2 B_C}{\lambda_{min}\big(\mathbf{\Sigma_1}\big)}  \stackrel{\ref{lem:DPofRS}}{\leq} \frac{\sigma_A^2B_C}{\sigma_A^2\bar{\lambda} + \sigma_g^2} = \frac{B_C}{\bar{\lambda} + \sigma_g^2/\sigma_A^2},
\end{align}


where in the last line, we assumed that $\|\mathbf{Q}(\mathcal{D})\mathbf{Q}(\mathcal{D})^\top -  \mathbf{Q}(\mathcal{D}')\mathbf{Q}(\mathcal{D}')^\top\|_F \leq B_C$ and used the minimum eigen value assumption of \Cref{lem:DPofRS}.
Following the assumptions in \Cref{lem:DPofRS} and according to \Cref{lem:AABBbound}, we can set $B_C = 2B_F B_S$ (we can use tighter bounds for $B_C$ when stronger assumptions exist about the difference between $\mathbf{Q}(\mathcal{D})$ and $\mathbf{Q}(\mathcal{D}')$. For example,  when as in \citep{lev2025gaussianmixingmechanismrenyi}, their difference is equal to one column of $\mathbf{Q}(\mathcal{D})$ (or one column of $\mathbf{Q}(\mathcal{D}')$) with bounded column norm $B_S$, we can set $B_C = B_S^2$). The bound $B_C = 2B_F B_S$ yields to:

\begin{align}\label{eq:eigenvaluesofH}
    \lambda_i(\mathcal{H}) \leq \frac{B_C}{\bar{\lambda} + \sigma_g^2/\sigma_A^2} = \frac{2 B_F B_S}{\bar{\lambda} + \sigma_g^2/\sigma_A^2} = \Gamma^{-1} \stackrel{\ref{lem:DPofRS}}{<} 1.
\end{align}

In a similar way, we can show that the following bounds hold for the eigen values of $\mathcal{J}$ and $\mathcal{Q}$ and get to:

\begin{align}\label{eq:eigenvaluesofJandQ}
    0 < \lambda_i(\mathcal{H}) \leq \Gamma^{-1}<1,\nonumber \\
    0<\lambda_i(\mathcal{J}) \leq \Gamma^{-1}<1,\nonumber \\
    0<\lambda_i(\mathcal{Q}) \leq \Gamma^{-1}<1.
\end{align}

Using the above result, we can also extend \cref{eq:egenvalue_comparison} as well to get to the following:

\begin{align}\label{eq:egenvalue_comparison_extended}
    0 < \lambda_i(\mathcal{H}) \leq \lambda_i(\mathcal{J}) \leq \Gamma^{-1}<1.
\end{align}

We can now use the above bounds to rewrite the divergence derived in \cref{eq:renyiddprime_simplified} as follows:

\begin{align}\label{eq:renyiddprime_simplified_rewritten}
D_{\alpha}(\mathcal{M}(\mathbf{Q}(\mathcal{D}))||\mathcal{M}(\mathbf{Q}(\mathcal{D}'))) & = \frac{r}{2(\alpha - 1)} \bigg ( \sum_{i=1}^{\tau^+} \log \frac{ (1+ \lambda_i(\mathcal{Q}))^{\alpha}}{(1+\alpha \lambda_i(\mathcal{Q}))} + \sum_{i=1}^{\tau^-} \log \frac{(1- \lambda_i(\mathcal{H}))^{\alpha}}{(1- \alpha \lambda_i(\mathcal{J}))}\bigg) \nonumber \\
& \stackrel{\ref{eq:egenvalue_comparison_extended}}{\leq} \frac{r}{2(\alpha - 1)} \bigg ( \sum_{i=1}^{\tau^+} \log \frac{ (1+ \lambda_i(\mathcal{Q}))^{\alpha}}{(1+\alpha \lambda_i(\mathcal{Q}))} + \sum_{i=1}^{\tau^-} \log \frac{(1- \lambda_i(\mathcal{J}))^{\alpha}}{(1- \alpha \lambda_i(\mathcal{J}))}\bigg) \nonumber \\
& \leq \frac{r}{2(\alpha - 1)} \bigg ( \sum_{i=1}^{\tau^+} \log \frac{ (1- \lambda_i(\mathcal{Q}))^{\alpha}}{(1-\alpha \lambda_i(\mathcal{Q}))} + \sum_{i=1}^{\tau^-} \log \frac{(1- \lambda_i(\mathcal{J}))^{\alpha}}{(1- \alpha \lambda_i(\mathcal{J}))}\bigg)
\nonumber \\
& \leq \frac{r \tau}{2(\alpha -1 )}\bigg[\alpha \log \big(1-\Gamma^{-1}\big) - \log \big(1-\alpha \Gamma^{-1}\big) \bigg]
\end{align}

where the second last inequality comes from the fact that function $f(t; \alpha) = \log \frac{(1-t)^\alpha}{1-\alpha t} - \log \frac{(1+t)^\alpha}{1+\alpha t}$ is always non-negative for $\alpha > 1$ and $\alpha t <1$ (in \Cref{eq:renyiddprime_simplified_rewritten}, $\alpha > 1$ and $\alpha \lambda_i(\mathcal{Q})< \alpha \Gamma^{-1}<1$). Also, in the last inequality, we have used the fact that function $f(t; \alpha) = \log \frac{(1-t)^\alpha}{1-\alpha t}$ is non-decreasing for $\alpha > 1$ and $\alpha t <1$. So we replaced the eigen values $\lambda_i$ with their upper bound $\Gamma^{-1}<1$.

\end{proof}

\subsection{Proof of \Cref{lem:DPofLora}}

\DPofLoRA*

\begin{proof}
    The proof is a direct result of the proof of \Cref{lem:DPofRS} above. First, we  note that the norm bound assumption for sample gradients $\|\nabla_{\mathbf{W}_\ell^t}l(x_i; \mathbf{W}_\ell^t)\|_F \leq C$ directly results in a similar norm bound for batch gradients:

    \begin{align}\label{eq:batchgradient_normbound}
        &\|\nabla_{\mathbf{W}_{\ell}^t} \mathcal{L}^t \|_F = \big\|\frac{1}{b}\sum_{i \in \mathcal{B}^t}\nabla_{\mathbf{W}_\ell^t}l(x_i; \mathbf{W}_\ell^t)\big\|_F \leq \frac{1}{b} \sum_{i \in \mathcal{B}^t}\big\|\nabla_{\mathbf{W}_\ell^t}l(x_i; \mathbf{W}_\ell^t)\big\|_F \leq C,
    \end{align}

    where data batches $\mathcal{B}^t$ are sampled from $\mathcal{D}$. Therefore, for $\mathcal{G}_\ell(\mathcal{D})= \sum_{t=0}^{T-1}\nabla_{\mathbf{W}_{\ell}^t} \mathcal{L}^t$ defined in the lemma, we have:

    \begin{align}\label{eq:H_normbound}
        &\|\mathcal{G}_\ell(\mathcal{D})\|_F = \big\|\sum_{t=0}^{T-1}\nabla_{\mathbf{W}_{\ell}^t} \mathcal{L}^t\big\|_F \leq TC.
    \end{align}

    We now find a sensitivity upper bound for $\max_{\mathcal{D}, \mathcal{D}'} \|\mathcal{G}_\ell(\mathcal{D}) - \mathcal{G}(\mathcal{D}')\|_F$, where $\mathcal{D}$ and $\mathcal{D'}$ are two neighboring dataset pairs by replacement, in the following. 
    
    \paragraph{\textbf{Sensitivity of $\mathcal{G}_\ell(\mathcal{D}) = \sum_{t=0}^{T-1}\nabla_{\mathbf{W}_{\ell}^t} \mathcal{L}^t$:}} Let us consider the adaptive composition $g(\mathcal{D})=f_2(\mathcal{D}, f_1(\mathcal{D}))$, for some functions $f_1, f_2, g$ and dataset $\mathcal{D}$. Let us assume that, at a fixed $f_1(\mathcal{D})$, function $f_2$ has sensitivity $S$ to one sample replacement in $\mathcal{D}$ and is $L$-Lipschitz continuous w.r.t. its second input (i.e., w.r.t. the output of $f_1(\mathcal{D})$). Then, the sensitivity of $g(\mathcal{D})$ to a sample replacement in $\mathcal{D}$ follows $\Delta g \leq S + L \Delta f_1$. We show this in the following:
    
    \begin{align}\label{lem:adaptivecomposition_sensitivity}
        \Delta g = \max_{\mathcal{D}, \mathcal{D}'}\|g(\mathcal{D}) - g(\mathcal{D}')\| &= \max_{\mathcal{D},\mathcal{D}'} \|f_2(\mathcal{D}, f_1(\mathcal{D}))-f_2(\mathcal{D}', f_1(\mathcal{D}'))\| \nonumber \\
        & = \max_{\mathcal{D},\mathcal{D}'} \big\|f_2(\mathcal{D}, f_1(\mathcal{D})) - f_2(\mathcal{D}', f_1(\mathcal{D}))- \big(f_2(\mathcal{D}', f_1(\mathcal{D}'))- f_2(\mathcal{D}', f_1(\mathcal{D}))\big)\big\| \nonumber \\
        & \leq \max_{\mathcal{D},\mathcal{D}'} \bigg(\big\|f_2(\mathcal{D}, f_1(\mathcal{D})) - f_2(\mathcal{D}', f_1(\mathcal{D}))\big\| + \big\| \big(f_2(\mathcal{D}', f_1(\mathcal{D}'))- f_2(\mathcal{D}', f_1(\mathcal{D}))\big)\big\| \bigg) \nonumber \\
        & \leq S + \max_{\mathcal{D},\mathcal{D}'} \big\| \big(f_2(\mathcal{D}', f_1(\mathcal{D}'))- f_2(\mathcal{D}', f_1(\mathcal{D}))\big)\big\| \nonumber \\
        & \leq S + L \max_{\mathcal{D},\mathcal{D}'} \big\|f_1(\mathcal{D}') -f_1(\mathcal{D})\big\| =  S + L \Delta f_1.
    \end{align}

    Having the above result, let us define $f_\ell ^t (\mathcal{D}, \mathbf{W}_{\ell}^t) = \mathbf{W}_{\ell}^t - \eta \nabla_{\mathbf{W}_{\ell}^t}\mathcal{L}^t  = \mathbf{W}_{\ell}^t - \frac{\eta}{b} \sum_{i \in \mathcal{B}^t}\nabla_{\mathbf{W}_{\ell}^t}l(x_i; \mathbf{W_\ell^t})$, i.e., an \SGD update at $\mathbf{W}_{\ell}^t$ with batch size $b$ on dataset $\mathcal{D}$. Therefore, we have an adaptive composition of the form $f_\ell ^t (\mathcal{D}, \mathbf{W}_{\ell}^t) = f_\ell ^t (\mathcal{D}, f_\ell ^{t-1} (\mathcal{D}, \mathbf{W}_{\ell}^{t-1}))$ that we saw above.  First, from the assumptions in \Cref{lem:DPofLora}, we have $\|\nabla_{\mathbf{W}_{\ell}^t}l(x_i; \mathbf{W}_{\ell}^t)\|\leq C$. Therefore, at a fixed $\mathbf{W}_\ell^t$, $f_\ell^t$ has sensitivity $\frac{2\eta C}{b}$ (by replacement). Second, remember from \Cref{lem:DPofLora} that loss function $l$ is $\beta$-smooth (or equivalently $\nabla_{\mathbf{W}_{\ell}^t} l$ is $\beta$-lipschitz). Therefore, $f_\ell^t(\mathcal{D}, \mathbf{W}_{\ell}^t)$ is $(1+\eta \beta)$-Lipschitz w.r.t. $\mathbf{W}_{\ell}^t = f_\ell ^{t-1} (\mathcal{D}, \mathbf{W}_{\ell}^{t-1})$. Therefore, using \Cref{lem:adaptivecomposition_sensitivity}, we can conclude that $\Delta f_\ell ^t \leq \frac{2 \eta C}{b} + (1+\eta \beta) \Delta f_\ell ^{t-1}$.  Furthermore, the sensitivity of function $f_\ell ^0 (\mathcal{D}, \mathbf{W}_{\ell}^0) = \mathbf{W}_{\ell}^0 - \frac{\eta}{b} \sum_{i \in \mathcal{B}^0}\nabla_{\mathbf{W}_{\ell}^0}l(x_i; \mathbf{W}_\ell^0)$ is equal to $\Delta f_\ell^0 = \frac{2 \eta C}{b}$ (remember $\mathbf{W}_\ell^0$ is fixed). Hence, for $f_\ell ^{T-1} (\mathcal{D}, \mathbf{W}_{\ell}^{T-1}) = \mathbf{W}_\ell^0-\eta \sum_{t=0}^{T-1}\nabla_{\mathbf{W}_{\ell}^t} \mathcal{L}^t$, we have:

    \begin{align}
        \Delta f_\ell ^{T-1} \leq \frac{2\eta C}{b} + \frac{2\eta C}{b} (1+\eta \beta) + \cdots + \frac{2\eta C}{b} (1+\eta \beta)^{T-1} = \frac{2 \eta C}{b} \frac{(1+\eta \beta)^{T}-1}{\eta \beta}.
    \end{align}

    As $f_\ell ^{T-1} (\mathcal{D}, \mathbf{W}_{\ell}^{T-1}) =  \mathbf{W}_\ell^0 - \eta \mathcal{G}_\ell(\mathcal{D})$, we have the following sensitivity for $\mathcal{G}_\ell$:

    \begin{align}\label{eq:H_sensitivitybound}
        \Delta \mathcal{G}_\ell = \frac{\Delta f_\ell ^{T-1}}{\eta}\leq \frac{2 C}{b} \frac{(1+\eta \beta)^{T}-1}{\eta \beta}.
    \end{align}

    Having found the above sensitivity for $\mathcal{G}_\ell$, we first analyze the privacy of \LoRA in each layer $\ell$ and then we compute the overall privacy for the fine-tuning procedure. Assuming $T$ gradient updates in the $\ell$-th layer (usually $T=EN/b$, where $E$ is the number of fine-tuning epochs, $N$ is the dataset size, and $b$ is the batch size), let us define the following matrices to be released:
    \begin{align}
        \mathcal{M}_{\ell}(\mathcal{D}) := \sum_{t=0}^{T-1}\nabla_{\mathbf{W}_{\ell}^t} \mathcal{L}^t (\mathcal{D})\mathbf{A}_{\ell}^0{^\top} + \mathbf{Z}_\ell = \mathcal{G}_\ell (\mathcal{D})+ \mathbf{Z}_\ell \in \mathbb{R}^{n \times r}, ~~~\ell\in\{1, 2, \ldots, L\}
    \end{align}

    where $\nabla_{\mathbf{W}_{\ell}^t} \mathcal{L}^t (\mathcal{D})$ means that the batch gradient $\nabla_{\mathbf{W}_{\ell}^t} \mathcal{L}^t$ was computed over dataset $\mathcal{D}$. Therefore, we can prove the privacy of \LoRA in the $\ell$-th layer by bounding $\mathcal{D}_\alpha \big(\mathcal{M}_{\ell}(\mathcal{D}) || \mathcal{M}_{\ell}(\mathcal{D}')\big)$, i.e., the R\'enyi divergence between $\mathcal{M}_{\ell}(\mathcal{D})$ and $\mathcal{M}_{\ell}(\mathcal{D}')$. To this end, we can now use the result of $\Cref{lem:DPofRS}$ directly to bound $\mathcal{D}_\alpha \big(\mathcal{M}_{\ell}(\mathcal{D}) || \mathcal{M}_{\ell}(\mathcal{D}')\big)$: we just need to set $B_F = TC$, and $B_S=\frac{2C}{b} \frac{(1+\eta \beta)^{T}-1}{\eta \beta}$ (based on \Cref{eq:H_normbound} and \Cref{eq:H_sensitivitybound}) into the \Cref{lem:DPofRS}, which results in:

    \begin{align}\label{eq:resultofdpofrs}
        \mathcal{D}_\alpha \big(\mathcal{M}_{\ell}(\mathcal{D}) || \mathcal{M}_{\ell}(\mathcal{D}')\big) \leq \frac{r \tau}{2(\alpha -1 )}\bigg[\alpha \log \big(1-\Gamma^{-1}\big) - \log \big(1-\alpha \Gamma^{-1}\big) \bigg],
    \end{align}
    with $\Gamma = \frac{\eta \beta}{4}\cdot \frac{b( \bar{\lambda} + r\sigma_g^2)}{ C^2 T ((1+ \eta \beta)^{T}-1)} > 1$. The above inequality provides the privacy guarantee for releasing $\mathcal{M}_\ell (\mathcal{D})$ for each of the $L$ layers. Simultaneous release of all the $L$ layers together, i.e., release of $\mathcal{M}(\mathcal{D}):=\{\mathcal{M}_{1}(\mathcal{D}), \ldots, \mathcal{M}_{L}(\mathcal{D})\big\}$, is still private as:

 \begin{align}\label{eq:overallrenyibound}
        \mathcal{D}_\alpha \big(\mathcal{M}(\mathcal{D}) || \mathcal{M}(\mathcal{D}')\big) = \sum_{\ell=1}^{L} \mathcal{D}_\alpha \big(\mathcal{M}_{\ell}(\mathcal{D}) || \mathcal{M}_{\ell}(\mathcal{D}')\big) \stackrel{\ref{eq:resultofdpofrs}}{\leq} \frac{r L\tau}{2(\alpha -1 )}\bigg[\alpha \log \big(1-\Gamma^{-1}\big) - \log \big(1-\alpha \Gamma^{-1}\big) \bigg],
    \end{align}

where the first equality comes from additivity of R\'enyi divergence for independent mechanisms. Finally, note that $\tilde{\mathbf{W}}_{\ell}^T (\mathcal{D}) = \mathbf{W}_{\ell}^0 - \eta \bigg( \sum_{t=0}^{T-1}\nabla_{\mathbf{W}_{\ell}^t} \mathcal{L}^t (\mathcal{D}) \mathbf{A}_{\ell}^0{^\top} + \mathbf{Z}_\ell \bigg) \mathbf{A}_{\ell}^0 = \mathbf{W}_{\ell}^0 - \eta \mathcal{M}_\ell (\mathcal{D}) \mathbf{A}_{\ell}^0 = \mathbf{W}_{\ell}^0 + \tilde{\mathbf{B}}_\ell^T \mathbf{A}_{\ell}^0$, which is a post-processing of $\mathcal{M}_\ell (\mathcal{D})$. Therefore, the privacy guarantee of releasing only $\{\tilde{\mathbf{W}}_{\ell}^T\}_{l=1}^L$ (without releasing the sketching matrices $\{\mathbf{A}_\ell^0\}_{\ell=1}^L$) can be obtained from the same bound \ref{eq:overallrenyibound}. This completes the proof of \Cref{lem:DPofLora}.
\end{proof}

\subsection{Proof of \Cref{cor:tCDPofLora}}
\tCDPofLoRA*

\begin{proof}
    The following proof is adopted from \citep{lev2025gaussianmixingmechanismrenyi} with minor changes. We know that $\Gamma>1$ and $\alpha \in (1, \Gamma)$. Let us define $k:=rLn$ and the difference function $d(k, \alpha, \Gamma)$ as:
    \begin{align}
        d(k, \alpha. \Gamma) := \frac{k\alpha}{2\Gamma^2} - \frac{k\alpha}{2(\alpha - 1)}\log(1-\Gamma^{-1}) + \frac{k}{2(\alpha - 1)}\log(1-\alpha\Gamma^{-1}).
    \end{align}
    Our goal is to find a subrange $(1,\zeta)$ of $(1,\Gamma)$ such that for every $\alpha \in (1,\zeta)$, we have $d(k, \alpha, \Gamma)\geq 0$. We next multiply function $d(k, \alpha, \Gamma)$ by $2 \Gamma^2 (\alpha - 1)$ and drop the constant $k=rLn>0$ to get to the following equivalent condition:

    \begin{align}
        e(\alpha) := \alpha(\alpha-1) - \alpha \Gamma^2\log(1-\Gamma^{-1}) + \Gamma^2 \log(1-\alpha\Gamma^{-1}) \geq 0, ~~~~\alpha \in (1, \zeta),
    \end{align}

where $\zeta$ is to be found. Note that at $\alpha=1$, $e(1)=0$. Also, we have:

 \begin{align}
        e'(\alpha) &= 2(\alpha - 1) - \Gamma^2 \log(1-\Gamma^{-1}) + \Gamma^2 \frac{-\Gamma^{-1}}{1-\alpha \Gamma^{-1}} \nonumber \\
        & = 2(\alpha - 1) - \Gamma^2 \log(1-\Gamma^{-1}) - \frac{\Gamma^2}{\Gamma-\alpha}.
    \end{align}

    By multiplying the equation above by $(\Gamma - \alpha) >0$, we get to the following quadratic function which has the same sign as $e'(\alpha)$. 

     \begin{align}
        H(\alpha):= (\Gamma - \alpha)e'(\alpha) &= 2(\alpha - 1)(\Gamma - \alpha) - \Gamma^2(\Gamma - \alpha) \log(1-\Gamma^{-1}) - \Gamma^2 \nonumber \\
        & = -2\alpha^2 + \bigg(1+2\Gamma + \Gamma^2 \log(1-\Gamma^{-1})\bigg)\alpha - \bigg(1+\Gamma + \Gamma^2 \log(1-\Gamma^{-1})\bigg)\Gamma.
    \end{align}

    We have the following discriminant for the function $H(\alpha)$ above:

    \begin{align}
        \Delta_H = \bigg(1+2\Gamma + \Gamma^2 \log(1-\Gamma^{-1})\bigg)^2 - 8\Gamma \bigg(1+\Gamma + \Gamma^2 \log(1-\Gamma^{-1})\bigg).
    \end{align}

    Therefore $H(\alpha)$ has two real roots:

    \begin{align}
        \alpha_{min/max} = \frac{\bigg(1+2\Gamma + \Gamma^2 \log(1-\Gamma^{-1})\bigg) \pm \sqrt{\Delta_H}}{4}.
    \end{align}

    Since the quadratic term of $H(\alpha)$ has the negative coefficient ``$-2$", $H(\alpha)$ is positive between $\alpha_{min}$ and  $\alpha_{max}$. For all $\Gamma>1$, we have $\alpha_{min}<1$. Furthermore, for $\Gamma>5/2$, we have $\alpha_{max} \in (1, \Gamma)$ and $\alpha_{max}>2\Gamma/5$. So if $\Gamma>5/2$, we can set $\zeta = \alpha_{max}$ and: for every $\alpha\in (1, \zeta)$, the functions $H(\alpha)$, $e(\alpha)$ and consequently $d(k, \alpha, \Gamma)$ will be positive. The proof is complete.

\end{proof}

\subsection{Proof of \Cref{lem:lora_final_forwardpass}}
\lorafinalforwardpass*

\begin{proof}
    
Our proof is directly based on the derivations in \citep{flora}. We first restate the following theorem from the work without restating its proof:

\begin{theorem}[\citet{flora}]
\label{thm:lora_dynamics}
Let \LoRA update matrices $\mathbf{A}_{\ell}$ and $B$ with \SGD as:

\begin{align}
    \mathbf{A}_{\ell}^{t+1} \gets \mathbf{A}_{\ell}^t - \eta \frac{\partial \mathcal{L}}{\partial \mathbf{A}_{\ell}^t} = \mathbf{A}_{\ell}^t - \eta \mathbf{B}_{\ell}^t{^\top} (\nabla_{\mathbf{W}_{\ell}^t} \mathcal{L}^t),
\end{align}

\begin{align}
    \mathbf{B}_{\ell}^{t+1} \gets \mathbf{B}_{\ell}^t - \eta \frac{\partial \mathcal{L}}{\partial \mathbf{B}_{\ell}^t} = \mathbf{B}_{\ell}^t - \eta (\nabla_{\mathbf{W}_{\ell}^t} \mathcal{L}^t)\mathbf{A}_{\ell}^t{^\top},
\end{align}

where $\eta$ is the learning rate. We assume $\|\sum_{t=0}^{\tau} \nabla_{\mathbf{W}_{\ell}^t}\mathcal{L}^t\|_F \leq L$ for every $\tau\leq T$ during training, implying that the model stays within a finite Euclidean ball from $\mathbf{W}_{\ell}^0$. In this case, the dynamics of $\mathbf{A}_{\ell}^t$ and $\mathbf{B}_{\ell}^t$ are given by
\begin{align}\label{eq:AB_dynamics}
    \mathbf{A}_{\ell}^t = \mathbf{A}_{\ell}^0 + \eta \mathbf{A}_{\ell}^0f_{\mathbf{A}_{\ell}}(t), ~~~ \mathbf{B}_{\ell}^t = \eta f_{\mathbf{B}_{\ell}}(t)\mathbf{A}_{\ell}^0{^\top},
\end{align}

where the forms of $f_{\mathbf{A}_{\ell}}(t)\in \mathbb{R}^{m \times m}$ and $f_{\mathbf{B}_{\ell}}(t) \in \mathbb{R}^{n \times m}$ are as follows for $t\geq 1$:

\begin{align}
    f_{\mathbf{A}_{\ell}}(t)&= -\eta \sum_{i=0}^{t-1} f_{\mathbf{B}_{\ell}}^{\top}(i)(\nabla_{W^i} \mathcal{L}^i), \\
    f_{\mathbf{B}_{\ell}}(t)&= -\sum_{i=0}^{t-1} (\nabla_{W^i} \mathcal{L}^i)(\eta f_{\mathbf{A}_{\ell}}^{\top}(i)+I),
\end{align}
and $f_{\mathbf{A}_{\ell}}(0) = f_{\mathbf{B}_{\ell}}(0) = 0$. In particular, $\|f_{\mathbf{A}_{\ell}}(t)\|_2 \leq \frac{\eta L^2 \big(1-(\eta^2 L^2)^t\big)}{1-\eta^2L^2}$ for every $t$.
\end{theorem}

Let us denote the forward-pass parameter after $T$ gradient steps on $\mathbf{A}_{\ell}^0$ and $\mathbf{B}_{\ell}^0$ as $\mathbf{W}_{\ell}^T = \mathbf{W}_{\ell}^0 + \mathbf{B}_{\ell}^T\mathbf{A}_{\ell}^T$. Therefore, we have:

\begin{align}\label{eq:forward_update}
    \mathbf{W}_{\ell}^T = \mathbf{W}_{\ell}^0 + \mathbf{B}_{\ell}^T\mathbf{A}_{\ell}^T &= \mathbf{W}_{\ell}^0 + (\eta f_B(T)\mathbf{A}_{\ell}^0{^\top})(\mathbf{A}_{\ell}^0 + \eta \mathbf{A}_{\ell}^0 f_A(T)) \nonumber \\
    & = \mathbf{W}_{\ell}^0 + \eta f_B(T)\mathbf{A}_{\ell}^0{^\top}\mathbf{A}_{\ell}^0 + \eta^2 f_B(T) \mathbf{A}_{\ell}^0{^\top} \mathbf{A}_{\ell}^0 f_A(T).
\end{align}

After substituting $f_B(T)$ into the second term in the above equation and noting that $f_A(T) \in \mathcal{O}(\eta)$, we get to:

\begin{align}
    \mathbf{W}_{\ell}^T = \mathbf{W}_{\ell}^0 + \mathbf{B}_{\ell}^T \mathbf{A}_{\ell}^T = \mathbf{W}_{\ell}^0 - \eta \bigg(\sum_{t=0}^{T-1} \nabla_{\mathbf{W}_{\ell}^t} \mathcal{L}^t \mathbf{A}_{\ell}^0{^\top} \mathbf{A}_{\ell}^0 \bigg ) + \mathcal{O}(\eta^3 ),
\end{align}

which completes the proof.

\end{proof}

\subsection{Proof of \Cref{lem:noise_var_kaiming}}
\noisevarkaiming*

\begin{proof}
The elements of $\mathbf{A}_{\ell}$ are from $\mathcal{U}(-b, b)$, which has variance $\sigma_{\mathbf{A}}^2=b^2/3$. Let $a_{i,j}$ denote the element in $i$-th row and $j$-th column of $\mathbf{A}_{\ell}$. Therefore, for all $i$ and $j$, $a_{i,j}$ has distribution $\mathcal{U}(-b, b)$. Let $B= \mathbf{A}_{\ell}^\top \mathbf{A}_{\ell}-\mathbb{I}_m$. Also, let $A_{i,:}$ and $A_{:,j}$ denote the $i$-th row and $j$-th column of $\mathbf{A}_{\ell}$, respectively. We have:

\begin{align}
    B_{i,i} = [\mathbf{A}_{\ell}^{\top} A]_{i,i}-1 = A_{:,i}^{\top} A_{:,i}-1 =  \|A_{:,i}\|_2^2 -1 = \big(\sum_{l=1}^r a_{l,i}^2 \big) -1
\end{align}

We know that $a_{l,i}$ is from $\mathcal{U}(-b,b)$ distribution. Hence, $a_{l,i}^2/b^2$ follows \texttt{Beta}(1/2, 1) distribution. Equivalently, $\mathbb{E}[a_{l,i}^2]= \texttt{Var}[a_{l,i}] = \sigma_{\mathbf{A}}^2 = b^2/3$ and $\texttt{Var}[a_{l,i}^2]= 4b^4/45$. Therefore, for $B_{i,i} = \sum_{l=1}^r a_{l,i}^2 -1$ ($i \in \{1,\ldots,m\}$),  which is the sum of $r$ independent such variables, we have: 

\begin{align} \label{eq:diagonal_mean_var_uniform}
    &\mathbb{E}[B_{i,i}] = \mathbb{E}\big[ \sum_{l=1}^r a_{l,i}^2\big] -1  = \frac{rb^2}{3} -1, \nonumber
    \\
    &\texttt{Var}[B_{i,i}] = \texttt{Var}[\sum_{l=1}^r a_{l,i}^2] = \frac{4rb^4}{45}.
\end{align}

Similarly, we find the mean and variance of the non-diagonal elements $B_{i,j} (i\neq j)$ of $B$. We have:

\begin{align}
    B_{i,j} = [\mathbf{A}_{\ell}^{\top} A]_{i,j} =  A_{:,i}^{\top} A_{:,j} = \sum_{l=1}^r a_{l,i} a_{l,j},
\end{align}

where $a_{l,i}$ and $a_{l,j}$ are independent. Therefore, we can derive that :

\begin{align}
    \mathbb{E}[a_{l,i} a_{l,j}] = 0 \\
    \texttt{Var}[a_{l,i} a_{l,j}] = \frac{5b^4}{45}.
\end{align}

Consequently, we can compute the mean and variance of the non-diagonal elements of $B$ ($i \neq j$):

\begin{align}\label{eq:nondiagonal_mean_var_uniform}
    &\mathbb{E}[B_{i,j}] = \mathbb{E}\big[ \sum_{l=1}^r a_{l,i}a_{l,j}\big] = 0, \nonumber 
    \\
    &\texttt{Var}[B_{i,j}] = \texttt{Var}[\sum_{l=1}^r a_{l,i}a_{l,j}] = \sum_{l=1}^r \texttt{Var}[a_{l,i}a_{l,j}] = \frac{5rb^4}{45}.
\end{align}

In the last line, we could take the \texttt{Var} into the summation because the $r$ summands are independent.

So far, we have computed the mean and variance of each entry in $B= \mathbf{A}_{\ell}^\top \mathbf{A}_{\ell}-\mathbb{I}_m \in \mathbb{R}^{m \times m}$. Now, for a given $\mathbf{q}\in \mathbb{R}^{1\times m}$, we have:

\begin{align} \label{eq:qB_uniform}
    \mathbf{q}\cdot B  = \sum_{l=1}^m \mathbf{q}_l B_{l,:},
\end{align}
where $B_{l,:}$ is row $l$ of $B$. Let $u_i$ denote the $i$-th element of $\mathbf{q}\cdot B$. Hence, for each element $u_i$ ($i \in \{1,\ldots, m\}$), we have:

\begin{align}\label{eq:mean_var_uniform}
\mathbb{E}[u_i] &= \mathbb{E}\bigg[ \sum_{l=1}^m \mathbf{q}_l B_{l,i}\bigg] =\mathbf{q}_i \mathbb{E}[B_{i,i}] + \sum_{l\neq i} \mathbf{q}_l \mathbb{E}[B_{l,i}] = (rb^2/3 - 1)\mathbf{q}_i= (r\sigma_{\mathbf{A}}^2 - 1)\mathbf{q}_i, \nonumber\\
\texttt{Var}[u_i] &= \texttt{Var}\bigg[ \sum_{l=1}^m \mathbf{q}_l B_{l,i}\bigg] =\texttt{Var}\bigg[ \sum_{l=1}^m \mathbf{q}_l  \sum_{t=1}^r a_{t,l} a_{t,i}\bigg] = \texttt{Var}\bigg[\mathbf{q}_i a_{t,i}^2 + \sum_{l\neq i, l=1}^m \mathbf{q}_l  \sum_{t=1}^r a_{t,l} a_{t,i}\bigg] \nonumber \\
& = \mathbf{q}_i^2 \texttt{Var}[a_{t,i}^2] + \sum_{l\neq i, l=1}^m \mathbf{q}_l^2  \sum_{t=1}^r \texttt{Var}[a_{t,l} a_{t,i}]  \nonumber \\
& = \frac{4rb^4}{45} \mathbf{q}_i^2 + \sum_{l\neq i} \frac{5rb^4}{45}\mathbf{q}_l^2  = \frac{5rb^4}{45} \|\mathbf{q}\|_2^2 - \frac{rb^4}{45}\mathbf{q}_i^2 = r\sigma_{\mathbf{A}}^4 (\|\mathbf{q}\|_2^2 - \frac{\mathbf{q}_i^2}{5}).
\end{align}

In the second last line we could take the \texttt{Var} into the summations because every pair of the $mr$ summands are uncorrelated (so variance of their sum is sum of their variances). 

Finally, according to \cref{eq:qB_uniform}, each element $u_i$ of $\mathbf{q}B$ is the sum of $m$ random variables, for which the Lindeberg's condition is also satisfied: for each $i$, as $m \to \infty$, $s_m^2 = r\sigma_A^4 (\sum_{l=1}^m \mathbf{q}_l^2 - \frac{\mathbf{q}_i^2}{5})\to \infty$ (assuming non-zero entries for $\mathbf{q}$). Note that $m$ is the dimension of $\mathbf{q}$, and $s_m$ is the sum of variances of the $m$ random variables, which we found in \cref{eq:mean_var_uniform}. Hence, $[|u_i - 0|>\epsilon s_m] \downarrow \emptyset$ as $m \to \infty$. Therefore, from \cref{thm:Lindeberg_CLT}, we also conclude that as $m \to \infty$, each element of $\mathbf{q}B$ approaches a Gaussian with the mean and variance found in \cref{eq:mean_var_uniform}. Therefore, we conclude that having an $\mathbf{A}_{\ell}$, where the elements of $\mathbf{A}_{\ell}$ are \emph{i.i.d} and from $\mathcal{U}(-b, b)$, then as $m \to \infty$, $\mathbf{q}\cdot (\mathbf{A}_{\ell}^\top \mathbf{A}_{\ell}-\mathbb{I}_m) \in \mathbb{R}^{1 \times m}$ approaches a Gaussian $\mathcal{N}\big((r\sigma_{\mathbf{A}}^2 - 1)~\mathbf{q}, ~r\sigma_A^4 (\|\mathbf{q}\|_2^2 ~ \mathbb{I}_m + \mathbf{D})\big)$, where $\mathcal{D}$ is diagonal and $\mathcal{D
}_{i,i} = \frac{-\mathbf{q}_i^2}{5}$. This completes the proof.

\end{proof}

\subsection{Proof of \Cref{lem:noise_var}}
\noisevar*
\begin{proof}
From the theorem's assumption, we know that elements of $\mathbf{A}_{\ell}$ are from $\mathcal{N}(0, \sigma_{\mathbf{A}}^2)$. Therefore, we can rewrite the product $\mathbf{q} \cdot (\mathbf{A}_{\ell}^\top \mathbf{A}_{\ell}-\mathbb{I}_m) \in \mathbb{R}^{1 \times m}$ as the following product:

\begin{align}\label{eq:reparam}
    \mathbf{q}\cdot (\sigma_{\mathbf{A}}^2 \mathbf{A}_{\ell}^\top \mathbf{A}_{\ell}-\mathbb{I}_m) \in \mathbb{R}^{1 \times m}
\end{align}
where \emph{the elements of $\mathbf{A}_{\ell}$ are now from standard normal distribution}. Let $a_{i,j}$ denote the element in $i$-th row and $j$-th column of this new $\mathbf{A}_{\ell}$. Therefore, for all $i$ and $j$, $a_{i,j}$ has distribution $\mathcal{N}(0, 1)$. Let $B=\sigma_{\mathbf{A}}^2 \mathbf{A}_{\ell}^\top \mathbf{A}_{\ell}-\mathbb{I}_m$. Also, let $A_{i,:}$ and $A_{:,j}$ denote the $i$-th row and $j$-th column of the new $\mathbf{A}_{\ell}$, respectively. We have:

\begin{align}
    B_{i,i} = \sigma_{\mathbf{A}}^2[\mathbf{A}_{\ell}^{\top} A]_{i,i}-1 = \sigma_{\mathbf{A}}^2 A_{:,i}^{\top} A_{:,i}-1 = \sigma_{\mathbf{A}}^2 \|A_{:,i}\|_2^2 -1 = (\sigma_{\mathbf{A}}^2 \sum_{l=1}^r a_{l,i}^2) -1
\end{align}

From \cref{eq:reparam}, we know that $a_{l,i}$ is from standard normal distribution. Hence, $a_{l,i}^2$ is a chi-squared with 1 degree of freedom: $a_{l,i}^2 \sim \mathcal{X}_1^2$. Therefore, $\sum_{l=1}^r a_{l,i}^2$, which is the sum of $r$ independent chi-squared variables with 1 degree of freedom, is a chi-squared with $r$ degrees of freedom: $\sum_{l=1}^r a_{l,i}^2 \sim \mathcal{X}_{r}^2$ (see \cref{thm:chisquared}). Therefore, for $i \in \{1,\ldots,m\}$, we have: 

\begin{align} \label{eq:diagonal_mean_var}
    &\mathbb{E}[B_{i,i}] = \mathbb{E}\big[\sigma_{\mathbf{A}}^2 \sum_{l=1}^r a_{l,i}^2\big] -1  = r \sigma_{\mathbf{A}}^2 -1, \nonumber
    \\
    &\texttt{Var}[B_{i,i}] = \texttt{Var}[\sigma_{\mathbf{A}}^2 \sum_{l=1}^r a_{l,i}^2] = \sigma_A^4 \texttt{Var}(\mathcal{X}_r^2) = 2r\sigma_A^4.
\end{align}

Similarly, we find the mean and variance of the non-diagonal elements $B_{i,j} (i\neq j)$ of $B$. We have:

\begin{align}\label{eq:nondiagonal}
    B_{i,j} = \sigma_{\mathbf{A}}^2[\mathbf{A}_{\ell}^{\top} A]_{i,j} = \sigma_{\mathbf{A}}^2 A_{:,i}^{\top} A_{:,j} = \sigma_{\mathbf{A}}^2 \sum_{l=1}^r a_{l,i} a_{l,j},
\end{align}

where $a_{l,i}$ and $a_{l,j}$ are independent and standard normal. Therefore, $a_{l,i} + a_{l,j} \sim \mathcal{N}(0,2)$. Similarly, $a_{l,i} - a_{l,j} \sim \mathcal{N}(0,2)$. So we can rewrite $a_{l,i} a_{l,j}$ as:

\begin{align}
    a_{l,i} a_{l,j} = \frac{1}{4}(a_{l,i} + a_{l,j})^2 - \frac{1}{4}(a_{l,i} - a_{l,j} )^2 = \frac{1}{2}z_1^2 - \frac{1}{2}z_2^2,
\end{align}
where $z_1$ and $z_2$ are from standard normal. Therefore, $a_{l,i} a_{l,j} = \frac{\nu_1 - \nu_2}{2}$, where $\nu_1, \nu_2 \sim \mathcal{X}_1^2$. Also, $a_{l,i} + a_{l,j}$ and $a_{l,i} - a_{l,j}$ are independent variables (because both $a_{l,i}$ and $a_{l,j}$ are Gaussian, independent, and have equal variance). Hence, $z_1$ and $z_2$ are independent, and likewise $\nu_1$ and $\nu_2$ are independent. We conclude that: 

\begin{align}\label{eq:cross_product}
a_{l,i} a_{l,j} = \frac{1}{2}(\nu_1 - \nu_2),
\end{align}

 where $\nu_1, \nu_2 \sim \mathcal{X}_1^2$, and are independent. 
 
 Now, let us assume $\nu_1, \nu_2 \sim \mathcal{X}_k^2$ (a more general case), and let $M_{\nu_1}(t) = \mathbb{E}[e^{t\nu_1}]$ be the moment generating function (MGF) of $\nu_1$. In this case, we know that $M_{\nu_1}(t) = M_{\nu_2}(t) = (1-2t)^{-\frac{k}{2}}$ (MGF of $\mathcal{X}_k^2$). Hence, $M_{\nu_1-\nu_2}(t) = M_{\nu_1}(t)\cdot M_{\nu_2}(-t) = (1-4t^2)^{-\frac{k}{2}} = \big( \frac{\frac{1}{4}}{\frac{1}{4}-t^2}\big)^{\frac{k}{2}}$, which is the MGF of a symmetric about origin \texttt{Variance-Gamma} distribution with parameters $\lambda = \frac{k}{2}, \alpha=\frac{1}{2}, \beta=0, \mu=0, \gamma = \frac{1}{2}$. Therefore, when $\nu_1, \nu_2 \sim \mathcal{X}_k^2$, then $\nu_1 - \nu_2$ has this distribution, which has mean $\mu + 2 \beta \lambda/\gamma^2=0$ and variance $2 \lambda (1+2 \beta^2/\gamma^2)/\gamma^2=4k$. 
 
 In \cref{eq:cross_product}, we had $k=1$, as we had $\nu_1, \nu_2 \sim \mathcal{X}_1^2$. Hence, based on the discussion above, we have:

 \begin{align}
     \mathbb{E}[a_{l,i} a_{l,j}] &= 0 \\
     \texttt{Var}[a_{l,i} a_{l,j}] &= \frac{1}{4}\texttt{Var}[\nu_1 - \nu_2] = \frac{4k}{4} = 1
 \end{align}

Consequently, based on \cref{eq:nondiagonal} and from the results above, we can compute the mean and variance of the non-diagonal elements of $B$ ($i \neq j$):

\begin{align}\label{eq:nondiagonal_mean_var}
    &\mathbb{E}[B_{i,j}] = \mathbb{E}\big[\sigma_{\mathbf{A}}^2 \sum_{l=1}^r a_{l,i}a_{l,j}\big] = \sigma_{\mathbf{A}}^2 \sum_{l=1}^r \mathbb{E}[a_{l,i}a_{l,j}] = 0, \nonumber 
    \\
    &\texttt{Var}[B_{i,j}] = \texttt{Var}[\sigma_{\mathbf{A}}^2\sum_{l=1}^r a_{l,i}a_{l,j}] = \sigma_A^4 \sum_{l=1}^r \texttt{Var}[a_{l,i}a_{l,j}] = r \sigma_A^4.
\end{align}

So far, we have computed the mean and variance of each entry in $B= \sigma_{\mathbf{A}}^2\mathbf{A}_{\ell}^\top \mathbf{A}_{\ell}-\mathbb{I}_m \in \mathbb{R}^{m \times m}$ in \cref{eq:diagonal_mean_var} and \cref{eq:nondiagonal_mean_var}. Now, for a given $\mathbf{q}\in \mathbb{R}^{1\times m}$, we have:

\begin{align} \label{eq:qB}
    \mathbf{q}\cdot B  = \sum_{l=1}^m \mathbf{q}_l B_{l,:},
\end{align}
where $B_{l,:}$ is row $l$ of $B$. Let $u_i$ denote the $i$-th element of $\mathbf{q}\cdot B$. Hence, for each element $u_i$ ($i \in \{1,\ldots, m\}$), we have:

\begin{align}\label{eq:mean_var}
\mathbb{E}[u_i] &= \mathbb{E}\bigg[ \sum_{l=1}^m \mathbf{q}_l B_{l,i}\bigg] =\mathbf{q}_i \mathbb{E}[B_{i,i}] + \sum_{l\neq i} \mathbf{q}_l \mathbb{E}[B_{l,i}] = (r\sigma_{\mathbf{A}}^2 - 1)\mathbf{q}_i, \nonumber\\
\texttt{Var}[u_i] &= \texttt{Var}\bigg[ \sum_{l=1}^m \mathbf{q}_l B_{l,i}\bigg] =\texttt{Var}\bigg[ \sum_{l=1}^m \mathbf{q}_l \sigma_{\mathbf{A}}^2 \sum_{t=1}^r a_{t,l} a_{t,i}\bigg] \nonumber \\
& = \sum_{l=1}^m \mathbf{q}_l^2\sigma_A^4 \sum_{t=1}^r \texttt{Var}[a_{t,l} a_{t,i}] = 2 \mathbf{q}_i^2 r\sigma_A^4 + \sum_{l\neq i} \mathbf{q}_l^2 r\sigma_A^4\nonumber\\
& = \mathbf{q}_i^2 r\sigma_A^4 + \sum_{l=1}^m \mathbf{q}_l^2 r\sigma_A^4 = r\sigma_A^4 (\mathbf{q}_i^2 + \|\mathbf{q}\|_2^2). 
\nonumber \\
\end{align}

In the second last line we could take the \texttt{Var} into the summations because every pair of the $mr$ summands are uncorrelated (so variance of their sum is sum of their variances). 

Finally, according to \cref{eq:qB}, each element $u_i$ of $\mathbf{q}B$ is the sum of $m$ random variables, for which the Lindeberg's condition is also satisfied: for each $i$, as $m \to \infty$, $s_m^2 = r\sigma_A^4 (\mathbf{q}_i^2 + \sum_{l=1}^m \mathbf{q}_l^2)\to \infty$ ($m$ is the dimension of $\mathbf{q}$, and $s_m$ is the sum of variances of the $m$ random variables, which we found in \cref{eq:mean_var}). Hence, $[|u_i - 0|>\epsilon s_m] \downarrow \emptyset$ as $m \to \infty$. Therefore, from \cref{thm:Lindeberg_CLT}, we also conclude that as $m \to \infty$, each element of $\mathbf{q}B$ approaches a Gaussian with the mean and variance found in \cref{eq:mean_var}. Therefore, we conclude that having an $\mathbf{A}_{\ell}$, where the elements of $\mathbf{A}_{\ell}$ are \emph{i.i.d} and from $\mathcal{N}(0, \sigma_{\mathbf{A}}^2)$, then as $m \to \infty$, $\mathbf{q}\cdot (\mathbf{A}_{\ell}^\top \mathbf{A}_{\ell}-\mathbb{I}_m) \in \mathbb{R}^{1 \times m}$ approaches a Gaussian $\mathcal{N}\big((r\sigma_{\mathbf{A}}^2 - 1)~\mathbf{q}, ~r\sigma_A^4 (\|\mathbf{q}\|_2^2 ~ \mathbb{I}_m + \mathbf{D})\big)$, which completes the proof.

\end{proof}

\subsection{Proof of \Cref{lem:total_var}}
\TVdist*
\begin{proof}
From \cref{eq:qB}, we had:

\begin{align}
    u_i = \sum_{l\neq i, l=1}^m \mathbf{q}_l B_{l,i} + \mathbf{q}_i B_{i,i},
\end{align}
where $B_{l,i} = \frac{1}{r} \mathbf{A}_{{\ell}_{:,l}}^\top A_{:,i} = \frac{1}{2r} \sum_{t=1}^r V_t$, where $V_t \sim \texttt{Variance-Gamma}(\nu, \alpha, \beta, \mu)$ with $\nu = \beta = \mu =0$ and $\alpha=\frac{1}{2}$. Also $B_{i,i} = \frac{1}{r} \mathbf{A}_{{\ell}_{:,l}}^\top A_{:,i} - 1 = \frac{\texttt{X}_r}{r} - 1$, where $X_r\sim \mathcal{X}_r^2$. Therefore, we can rewrite the equation above for $u_i$ as:

\begin{align}\label{eq:ui_detailed}
    u_i & = \sum_{l\neq i, l=1}^m  \frac{\mathbf{q}_l}{r} \sum_{t=1}^r a_{t,l} a_{t,i} + \frac{\mathbf{q}_i}{r} \sum_{t=1}^r (a_{t,i}^2-1)= \sum_{t=1}^r \underbrace{\frac{1}{r}\bigg(\mathbf{q}_i(a_{t,i}^2-1) + \sum_{l\neq i, l=1}^m \mathbf{q}_la_{t,l}a_{t,i}\bigg)}_{V_t}\\
    &= \sum_{t=1}^r \underbrace{\frac{1}{r}\bigg( \mathbf{q}_i(\texttt{X}_{t} - 1) + \sum_{l\neq i, l=1}^m \frac{\mathbf{q}_l}{2} \texttt{Y}_{t,l} \bigg)}_{V_t}.
\end{align}

In order to be able to use the \Cref{thm:tv_distance}, each pair of the summands should be \textbf{independent}. In the left summation in \Cref{eq:ui_detailed}, there are $mr$ summands that are uncorrelated, but some of them are not independent, e.g., $\frac{\mathbf{q}_1}{r} a_{1,1}a_{1,i}$ and $\frac{\mathbf{q}_2}{r} a_{1,2}a_{1,i}$ are dependent (although they are uncorrelated). So, on the right, we rewrote the left summation as the sum of $r$ summands $V_t$ that are independent. 
In the last line, we again rewrote $V_t$ as the summation of a (centered) chi-squared variable with one degree of freedom ($\texttt{X}_{t} \sim \mathcal{X}_1^2$) and other variables $\texttt{Y}_{t,l}$ where $\texttt{Y}_{t,l} \sim \texttt{Variance-Gamma}(\nu, \alpha, \beta, \mu)$ with $\nu = \beta = \mu =0$ and $\alpha=\frac{1}{2}$ (we know this from \Cref{eq:cross_product}) based on the parameterization used in \citep{gaunt2024}. Hence, $\texttt{Y}_{t,l}$ has mean $0$ and variance $4$ and $(\texttt{X}_{t} - 1)$ has mean $0$ and variance $2$. Also, $\texttt{Y}_{t,l}$ and $\texttt{Y}_{t,l'}$ are not independent, but are uncorrelated.

Now, we can consider $u_i$ as the weighted sum of these $r$ independent summands $\{V_t\}_{t=1}^r$ and apply \Cref{thm:tv_distance} to it as follows. In order to bound the TV distance between the distribution of $u_i$ and $\mathcal{N}(0, \frac{\|\mathbf{q}\|_2^2 + \mathbf{q}_i^2}{r})$, we have to use \cref{thm:tv_distance} and \cref{eq:lyapounov_ratios}: we have to find the third Lyapounov ratio $L_3 = \frac{\sum_{t=1}^{r} \mathbb{E}[|V_t|^3]}{s_r^3} = \frac{\sum_{t=1}^{r} \mathbb{E}[|V_t|^3]}{\big(\sum_{t=1}^{r} \texttt{Var}[V_t]\big)^{3}} = \frac{\sum_{t=1}^{r} \mathbb{E}[|V_t|^3]}{\big(\sum_{t=1}^{r} \mathbb{E}[V_t^2]\big)^{3}}$, where $V_t$ is each of the $r$ summands in \cref{eq:ui_detailed}.

First we note that, 

\begin{align}\label{eq:denominator}
    s_r^3 = (s_r^2)^{\frac{3}{2}}= \bigg(\sum_{t=1}^r \texttt{Var}[V_t]\bigg)^{\frac{3}{2}} = \bigg( \sum_{t=1}^r \frac{1}{r^2}\big( 2\mathbf{q}_i^2 + \sum_{l \neq i, l=1}^m \mathbf{q}_l^2\big)\bigg)^{\frac{3}{2}} = \bigg(\frac{\|\mathbf{q}\|_2^2 + \mathbf{q}_i^2}{r}\bigg)^{\frac{3}{2}} \geq  \frac{\|\mathbf{q}\|_2^3}{r \sqrt{r}},
\end{align}

where the third equality holds because the $m$ summands in $V_{t}$ are uncorrelated.  

Now, we find the numerator $\sum_{i=1}^{r} \mathbb{E}[|V_t|^3]$ of the ratio. We leverage the Minkowski inequality in the $L^3$ space: suppose $V = \sum_{t=1}^r X_t$, i.e., the sum of $r$ random variables. We have: 

\begin{align}
    (\mathbb{E} [|V|^3])^{\frac{1}{3}} \leq \sum_{t=1}^r (\mathbb{E} [|X_t|^3])^{\frac{1}{3}}.
\end{align}

We apply this inequality to each of the $V_t$ terms in \Cref{eq:ui_detailed}. For the first summand in $V_t$, we have:

\begin{align}
    \mathbb{E}[|\frac{\mathbf{q}_i}{r} (\texttt{X}_t - 1)|^3] \leq \frac{9 |\mathbf{q}_i|^3}{r^3},
\end{align}
considering the fact that the third absolute central moment of a degree-1 chi-squared variable ($\texttt{X}_t$) is smaller than 9. Also, from \citep{gaunt2024}, we know that for $\texttt{Y}_{t,l} \sim \texttt{Variance-Gamma}(\nu, \alpha, 0, 0), \mathbb{E}[|\texttt{Y}_{t,l}|^r]= \frac{2^r}{\sqrt{\pi} \alpha^r} \frac{\Gamma(\nu + (r+1)/2) \Gamma((r+1)/2)}{\Gamma(\nu + 1/2)}$. Therefore, for $\texttt{Y}_{t,l} \sim \texttt{Variance-Gamma}(0, \frac{1}{2}, 0, 0), \mathbb{E}[|\texttt{Y}_{t,l}|^3]= \frac{2^6}{\pi}$. Therefore, for the other summands, we have:

\begin{align}
    \mathbb{E}[|\frac{\mathbf{q}_l}{2r} \texttt{Y}_{t,l}|^3] = \frac{|\mathbf{q}_l|^3}{8 r^3}\cdot \frac{2^6}{\pi}.
\end{align}

Consequently, from the last two equalities and the Minkowski inequality, we have:

\begin{align}
    (\mathbb{E}[|V_t|^3])^{\frac{1}{3}} \leq (\frac{9 |\mathbf{q}_i|^3}{r^3})^{\frac{1}{3}} + \sum_{l\neq i, l=1}^m (\frac{8 |\mathbf{q}_l|^3}{\pi r^3})^{\frac{1}{3}} \leq \sum_{l=1}^m (\frac{9 |\mathbf{q}_l|^3}{r^3})^{\frac{1}{3}} = \sum_{l=1}^m \frac{3^{2/3}|\mathbf{q}_l|}{r} = \frac{3^{2/3}\|\mathbf{q}\|_1}{r}.
\end{align}

Therefore, we have:
\begin{align}
    \mathbb{E}[|V_t|^3] \leq  \frac{9\|\mathbf{q}\|_1^3}{r^3} .
\end{align}

Finally, we find the numerator of the Lyapounov ratio:

\begin{align}\label{eq:numerator}
    \sum_{t=1}^r \mathbb E [|V_t|^3] \leq \sum_{t=1}^r \frac{9\|\mathbf{q}\|_1^3}{r^3}  = \frac{9\|\mathbf{q}\|_1^3}{r^2}   
\end{align}




Therefore, for the sum $u_i$ in \cref{eq:ui_detailed}, we can find the third Lyapounov ratio using \Cref{eq:numerator} and \Cref{eq:denominator}, and based on \cref{thm:tv_distance}, we have:

\begin{align}\label{eq:tv}
    \|Q_m(x)-\Phi(x)\|_{TV} \leq C_D L_3 \leq C_D \cdot \frac{9}{ r^2} \|\mathbf{q}\|_1^3 \cdot \frac{r\sqrt{r}}{\|\mathbf{q}\|_2^3} =  \frac{9C_D}{\sqrt{r}} \bigg(\frac{\|\mathbf{q}\|_1}{\|\mathbf{q}\|_2}\bigg)^3,
\end{align}

where $C_D$ is a constant, which depends only on $D$, where $D$ is an upperbound for the KL divergence between each of the $r$ summands $V_t$ in \cref{eq:ui_detailed} and a Gaussian with the same mean $0$ and variance $\frac{\|\mathbf{q}\|_2^2 + \mathbf{q}_i^2}{r}$ ($D$, and consequently $C_D$, are expected to be small because $V_t$ comprises of $m-1$ \texttt{Variance-Gamma} random variables with mean zero as well as one centered chi-square variable with mean zero, and $m$ is large. So the distribution of $V_t$ is expected to be close to $\mathcal{N}(0, \frac{\|\mathbf{q}\|_2^2 + \mathbf{q}_i^2}{r})$). 
Therefore:

\begin{align}
    \|Q_m(x)-\Phi(x)\|_{TV} \in \mathcal{O} \bigg(\frac{1}{\sqrt{r}}\bigg).
\end{align}

It is noteworthy that if the $mr$ summands in \Cref{eq:ui_detailed} were all independent (we saw that some pairs were just uncorrelated), we could consider $u_i$ as the sum of $mr$ independent summands (instead of \Cref{eq:ui_detailed} writing $u_i$ as the sum of $r$ independent variables $\{V_t\}_{t=1}^r$). In that case, we would be able to bound the numerator of the Lyapounov ratio with $\frac{9\|\mathbf{q}\|_3^3}{r^2}$ (instead of $\frac{9\|\mathbf{q}\|_1^3}{r^2}$ in \Cref{eq:numerator}). This better bound would enable us to prove that $\|Q_m(x)-\Phi(x)\|_{TV} \leq \frac{9C_D}{\sqrt{r}} \big(\frac{\|\mathbf{q}\|_3}{\|\mathbf{q}\|_2}\big)^3$ and consequently $\|Q_m(x)-\Phi(x)\|_{TV} \in \mathcal{O} \bigg(\frac{1}{\sqrt{mr}}\bigg)$ (with an additional assumption of $0<c<|\mathbf{q}_i|<C$, for all $i$).

\end{proof}


\end{document}